\documentclass{amsart}
\usepackage{amsthm}
\usepackage{amsmath}
\usepackage{amssymb}
 \usepackage{graphicx}
\newtheorem{theorem}{Theorem}
\newtheorem{lemma}[theorem]{Lemma}
\newtheorem{corollary}[theorem]{Corollary}
\newtheorem{prop}{Proposition}
\DeclareGraphicsExtensions{.PNG} 
\begin{document}

\title{High Dimensional Spaces, Deep Learning and Adversarial Examples}
\author{Simant Dube}
\date{January 1, 2018 (version 1)}
\thanks{Email: simantdube@iitdalumni.com}
\address{{\rm Simant Dube is a computer scientist working in San Francisco area and is a former academic.}}
\email{simantdube@iitdalumni.com}

\begin{abstract}
In this paper, we analyze deep learning from a mathematical point of view and derive several novel results. The results are based on intriguing mathematical properties of high dimensional spaces. We first look at perturbation based adversarial examples and show how they can be understood using topological and geometrical arguments in high dimensions. We point out mistake in an argument presented in prior published literature, and we present a more rigorous, general and correct mathematical result to explain adversarial examples in terms of topology of image manifolds. Second, we look at optimization landscapes of deep neural networks and examine the number of saddle points relative to that of local minima. Third, we show how multiresolution nature of images explains perturbation based adversarial examples in form of a stronger result. Our results state that expectation of $L_2$-norm of adversarial perturbations is $O\left(\frac{1}{\sqrt{n}}\right)$ and therefore shrinks to 0 as image resolution $n$ becomes arbitrarily large. Finally, by incorporating the parts-whole manifold learning hypothesis for natural images, we investigate the working of deep neural networks and root causes of adversarial examples and discuss how future improvements can be made and how adversarial examples can be eliminated.
\end{abstract}

\maketitle

\section{Introduction}
In last decade, there has been proliferation of applications of deep neural networks in the general field of Artificial Intelligence (AI) and specifically in computer vision, speech recognition and natural language understanding. In AI community there has been debate centered around need for greater mathematical rigor and understanding of deep learning. Though there has been significant progress in practical techniques based on impressive trial-and-error empirical work, theory has been lagging behind practice. Can there be simple mathematical results which shed light on how deep learning works? How can one understand shortcomings of present day deep learning which can pave way for future work?

In machine learning theory, there are several very well-known fundamental theoretical results. In this paper, we present several novel results on adversarial examples, optimization landscape, local minima and image manifolds, that provide mathematical rigor to the specific field of deep neural networks which operate in very high dimensions. We apply results from very high dimensional mathematical spaces to deep learning which is the primary goal of the paper.

There has been already significant research in unraveling how and why deep learning works. In~\cite{loss_surface_choromanska}, under certain assumptions and using results from random matrix theory applied to spin-glasses, authors evaluate loss surfaces of multi-layer feed forward neural networks. Their results indicate that there is a layered structure of critical points. Near global minimum, most of critical points are local minima. In higher bands, we start seeing saddle points of increasing index and the probability of finding local minima decreases exponentially. In~\cite{yann_dauphin}, based on evidences from several directions, such as statistical physics, random matrix theory and experimental work, strong thesis is presented which states that deep networks don't suffer from local minima problem and instead suffer from saddle points which can give illusory appearance of local minima. In~\cite{loss_surface_nguyen}, under certain assumptions which includes having very large and wide neural networks, it is shown that local minima are almost always close to the global minimum. In~\cite{levent_sagun}, empirical work to look at distribution of eigen-values of the Hessian matrix of the loss function is presented on simple examples and it is observed that the Hessian is very singular for these examples. In~\cite{poggio_2, poggio_3}, singularity of the Hessian is shown for underdetermined overparameterized systems.

Lot of interesting work has been done in the area of adversarial examples, see~\cite{explaining_goodfellow,kurakin_physical_world,fool_nguyen,intriguing_szegedy}. See survey in~\cite{adversarial_survey}. There are two kinds of adversarial examples. First kind is perturbation based in which an imperceptible perturbation is added to an image to change the output of the deep network, see~\cite{intriguing_szegedy}. In the second kind, images which are unrecognizable by humans are classified by deep networks with high probability, see~\cite{fool_nguyen}. See Figure~\ref{fig_adv_examples} for adversarial examples. Adversarial examples tell us something fundamental about the way present day deep networks work. The hypothesis presented in~\cite{jo_bengio} which states that CNNs are learning superficial cues rather than high-level semantic abstractions matches with the conclusions in this paper. Adversarial examples are of great practical significance too as they can pose risks to real world applications of deep learning, see~\cite{kurakin_physical_world}. For theoretical results on adversarial examples see~\cite{hein_andriushchenko}, where instance-specific lower bounds on the norm of the input manipulation required to change the classifier decision are given using theorems from calculus, under the assumption that the classifier is continuously differentiable. In this paper, the approach is quite different and we give geometrical proofs for arbitrary manifold geometries which show direct relationship of the norm to the input dimension using properties of high-dimensional spaces in the most general case without any differentiability assumptions.

In this paper, we point out mistake in an argument presented in a published paper in 2015 by Goodfellow et al., see reference~\cite{explaining_goodfellow}, where authors present argument to explain adversarial examples for linear model $w^T x$.  Problem in reasoning behind their argument has been independently highlighted earlier in~\cite{tanay_griffin}. The authors in~\cite{explaining_goodfellow} argue that as feature dimensionality $n$ of the linear model increases, one can increase total perturbation amount larger and larger, while keeping $L_\infty$-norm constant and imperceptibly small, till it reaches a desired target activation value to flip the classification result of the linear model. The problem in this argument is the assumption that desired target value remains constant and is independent of $n$. But as dimensionality increases, we are working in different spaces and the norms of all vectors change, so it becomes a moving target. The norms of weight vector $w$ and $x$ increase too. Therefore increasing $n$ does not help in finding adversarial examples for most samples as they are far away from the decision boundary. In order to generate adversarial examples, we will be forced to relax $L_\infty$-norm constraint to a higher value which will depend on the distance of sample from the decision boundary and these distances either will have no upper bound or could be very large (depending on feature value range) in general and therefore perturbation will be mostly not small. Consequently, the generalization of the argument in~\cite{explaining_goodfellow} to deep networks is not valid. In Theorem~\ref{linearity_theorem}, we state that the linear model does not suffer from adversarial examples.

In this paper, we present general, rigorous and correct mathematical results which work on any bounded image manifold. They work on manifolds carved out by deep networks through piecewise linear approximation as special case, provided they are bounded manifolds.

We present our results restricting ourselves to image classification problem in computer vision. Overview of our paper is as follows. Our first result in Section~\ref{section_adversarial_negatives} explains why one can generate perturbation based adversarial examples in deep learning. Then in Section~\ref{section_sgd} we look at the question why deep learning often works well in practice without getting stuck in local minima when using stochastic gradient descent (SGD). To understand that we make use of results about bounds on number of critical points of normal random polynomials. In Section~\ref{section_statistics_natural} we use the Manifold Learning hypothesis and statistics of natural images to understand the nature of high-dimensional image manifolds which allows us to make our mathematical results stronger. This also provides insights into the second kind of adversarial examples which are random looking noisy images giving high confidence outputs. In Section~\ref{section_future}, we present results to solve the problem of adversarial examples. We first investigate complexity of surfaces of image manifolds and apply empirical results from~\cite{lid_paper_deep} along with theoretical results in this paper to understand the root causes of adversarial examples. Based on recent work in~\cite{capsule_networks} on capsule networks, we consider Parts-Whole Manifold Learning Hypothesis to understand limitations of the present day deep learning, which shows the way for elimination of adversarial examples and for future improvements in deep learning and in the general field of AI.

\begin{figure}[t]
    \centering
		\includegraphics[width=\textwidth]{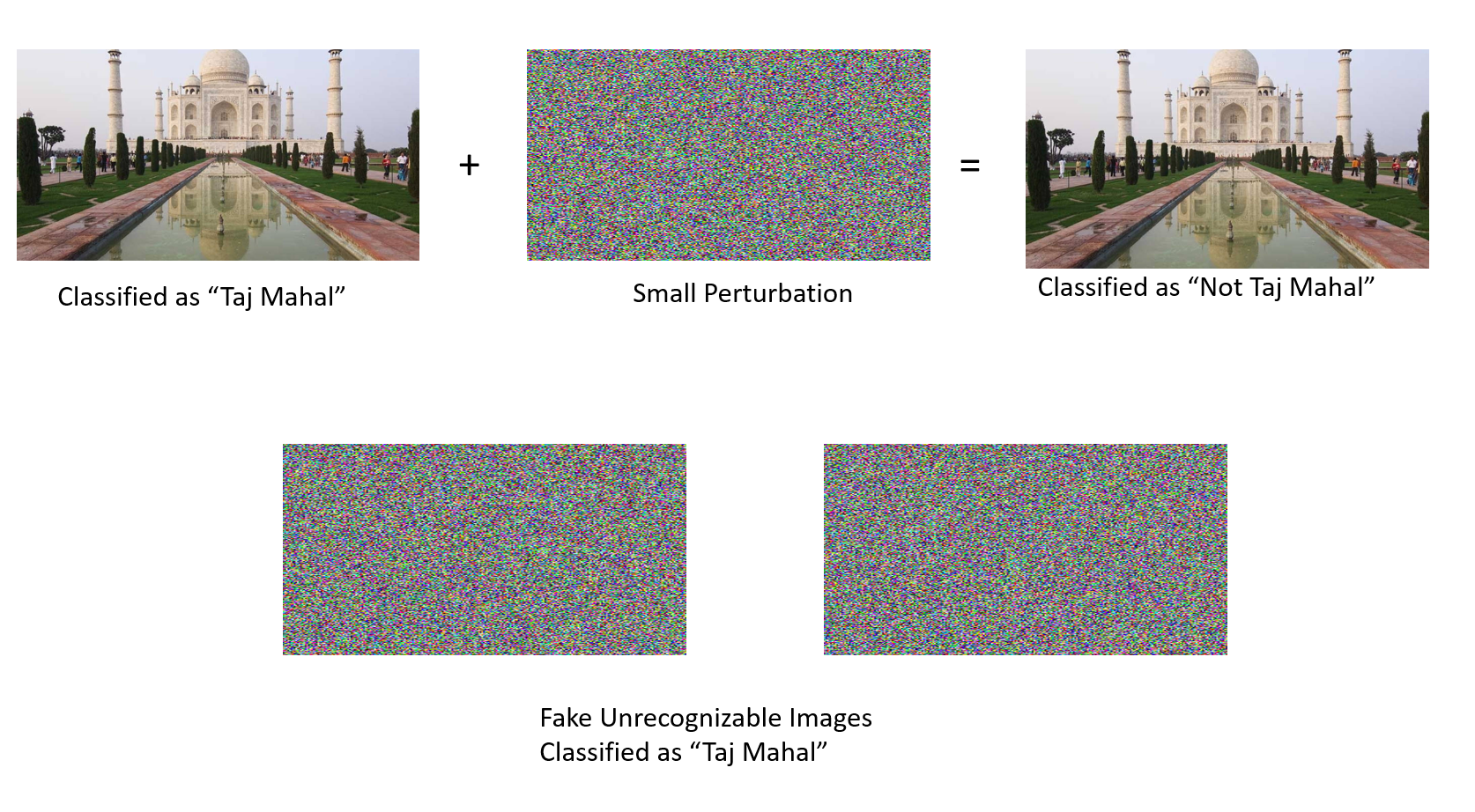}
    \caption{Deep learning suffers from adversarial examples. {\em Top}, a small perturbation image can be added to any image which makes the network change its result, see~\cite{intriguing_szegedy}. {\em Bottom}, fake unrecognizable images can be created which give high confidence classification outputs, see~\cite{fool_nguyen}.}
    \label{fig_adv_examples}
\end{figure}

\section{Adversarial Examples}
\label{section_adversarial_negatives}
It has been shown that deep learning suffers from the problem of adversarial examples, see Figure~\ref{fig_adv_examples}. One can slightly perturb a sample in such a way that the output of deep network changes. In case of images, the perturbed image is visually indistinguishable from the original.

First let us define the notion of an image and that of a deep learning system.

{\bf Definition:} An {\em infinite resolution\/} gray-scale image $I$ is a function on 2-D unit square, $f:[0,1]^2 \rightarrow \mathbb{R}$. A {\em finite resolution\/} version of the image $I$ is obtained by approximating $f$ in a grid of $n = r \times c$ pixels. An {\em image class\/} is a set of infinite resolution images which belong to a semantic category.

In practice, we will have finite resolution versions of the image class where particular resolution is constrained by memory and computing power. As technology progresses, images will have increasing resolutions. This has significant impact on mathematical analysis of deep learning as will be shown in this paper.

{\bf Definition:}
A {\em deep learning image classification system\/} is a 4-tuple $D = (A,T,S,Q)$ where:
\begin{itemize}
\item
$A$ is the deep neural network.
\item
$T$ is the training data for $k \geq 1$ {\em positive classes\/} $C = \{C_1, C_2, \ldots, C_k\}$ and the {\em negative class\/} $C_{k+1}$. $T$ consists of images at some finite resolution $n = r \times c$. The negative class contains sample images belonging to none of the $k$ positive classes.
\item
$S = \{G_i \subset \mathbb{R}^n| i \in \{1,2,\ldots,k\}\}$ is the set of {\em ground truth compact manifolds\/} for $k$ classes, where for all $i$, $G_i$ contains all images at resolution $n$ belonging to class $C_i$.
\item
$Q = \{P_i \subset \mathbb{R}^n| i \in \{1,2,\ldots,k\}\}$ is the set of {\em trained compact manifolds\/} for $k$ classes. Once $A$ has been trained using $T$, for all $i$, $P_i= \{ x \in \mathbb{R}^n | A(x) = C_i\}$ is the trained manifold for class $C_i$ and the goal of the training is to make it approximate $G_i$ as closely as is practically possible.
\end{itemize}

The type of set is chosen to be a manifold rather than an arbitrary set to emphasize that we are dealing with semantically meaningful natural image classes. The idea of image manifolds has been found to be useful in computer vision, see~\cite{srivastava_manifold, manifold_book}. Given a natural image, assumption that there is a locally Euclidean neighborhood around it in the set is a reasonable one~\cite{manifold_book}. Bounded dynamic range of image class implies that manifold is bounded. The greyness value of any image pixel can not shoot off to infinity. And if there is a convergent sequence of images, then the limiting image should be included in the set too, making the manifold compact.  Though even if the mathematical results can hold for arbitrary sets, we will see in later sections that restricting to manifolds is conceptually helpful in understanding natural images.

Let $x, y \in \mathbb{R}^n$ be a randomly selected image and its adversarial example, respectively. Let $y = x + p$ where $p \in \mathbb{R}^n$ is the perturbation image. We prove our first result which shows why it becomes easier to create a visually distinguishable adversarial example as resolution of images increases. The expected value of $L_2$-norm of the perturbations becomes very small for high-dimensional images.

It should be noted that the following results are applicable to any machine learning model provided conditions for the theorems are met. In classical machine learning too, which approximate ground truth manifolds by trained manifolds, as the dimensionality of input features will increase, so will the problem of adversarial examples. Since deep learning takes raw data as the input which has high dimensions, it is particular relevant to deep neural networks.

We have defined adversarial example in terms of any $x$ which gets positively classified. One could have focused only on correctly and positively classified samples. Since underlying practical assumption is to work with deep networks which yield high accuracy, we restrict to the mathematically simpler case of positively classified samples. A positively classified sample is very likely to be correctly classified in such high accuracy networks. An {\em adversarial example\/} is one which makes the output of the deep network {\em change\/} with visually insignificant perturbation.

Even if the trained manifolds are identical to the ground truth manifolds, note that there will be always images at the surface of these manifolds for which the output of the deep network changes on slight perturbation and they will always have adversarial examples by this definition. But we expect this only to be true only for borderline images and not for almost every image, see Figure~\ref{fig_two_manifolds}. What will be truly intriguing if we can mathematically prove that almost every image happens to have this property of having an adversarial example. We will show that this indeed follows from intriguing properties of high-dimensional spaces.

\begin{figure}[t]
    \centering
		\includegraphics[width=\textwidth]{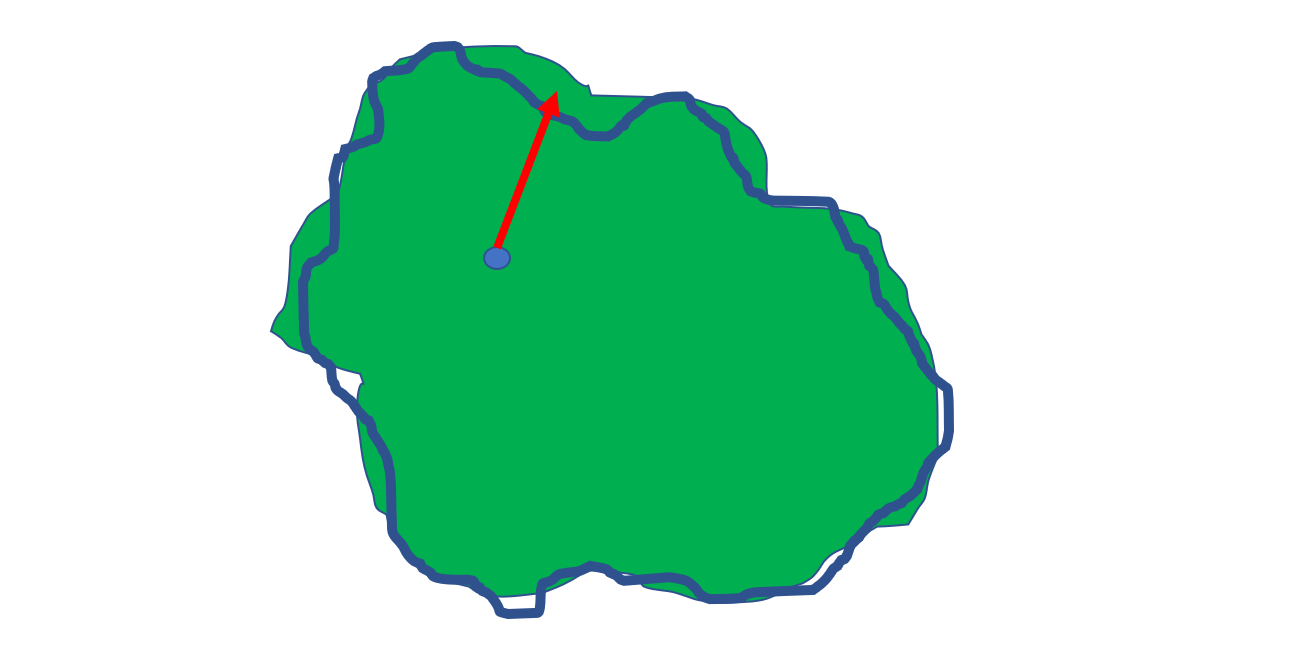}
    \caption{Consider the trained manifold as marked by the blue boundary, and define adversarial image as one which makes the decision of neural network change. The perturbation is shown by red arrow. We should expect this to be true only for borderline images for visually imperceptible perturbations. The ground truth manifold is shown in green. We will examine properties of the ground truth manifold later in the paper.}
    \label{fig_two_manifolds}
\end{figure}

\subsection{Image Manifolds which are $n$-balls}
Suppose one can perturb images very slightly to get adversarial examples and we were to bound this perturbation amount. An interesting question to ask is that given an image how probable it is to be able to perturb it within this bound and successfully get an adversarial example. For that, we first need the following Lemma.

\begin{lemma}\label{lemma_sphere}
Consider any image classification problem at any finite resolution $n = r \times c$. For any random sample $x \in \mathbb{R}^n$ which is classified positively by the deep network $A$ for one of the classes, let $y \in \mathbb{R}^n$ be the closest adversarial example where $A(x) \neq A(y)$ and let $y = x + p$ where $p \in \mathbb{R}^n$ is the perturbation image. Denote the perturbation $L_2$ norm $\left\Vert p \right\Vert$ for $x$ by $P(x)$.

Further, assume that trained manifold $B \subset \mathbb{R}^n$ for the image class which $x$ belongs to, is topologically an $n$-ball of radius $R$
\[B = \{ x \in \mathbb{R}^n | \left\Vert x \right\Vert \leq R\}\]

Let $\alpha$, where $0 < \alpha \leq 1$, be the relative perturbation bound with respect to the radius $R$.

Then,
\begin{enumerate}
\item[1.]
\[\mathop{\rm{Prob}}_{x \in B} \left\{ \left\Vert P(x) \right\Vert \leq \alpha R \right\} =  1 - (1-\alpha)^n\]
\item[2.]
\[\mathop{\mathbb{E}}_{x \in B} \left( \left\Vert P(x) \right\Vert \right) = \frac{R}{n+1}\]
\end{enumerate}
\end{lemma}

\begin{proof}
Let the surface area of the $n$-ball of radius $r$ be $S(n,r)$. The closest adversarial example $y$ for $x$ will be just outside surface of $B$ and its distance from $x$ will be just above the shortest distance of $x$ from the surface of the ball. 

For any given $w$ where $0 < w \leq R$, points in the ball will have adversarial examples within distance $w$ if they are within the outermost spherical shell of width $w$. We calculate the ratio of the volume contained in outermost spherical shell to the total volume, 

\begin{equation*}
\begin{split}
\mathop{\rm{Prob}}_{x \in B} \left\{ \left\Vert P(x) \right\Vert  \leq \alpha R \right\} & = \frac{\int_{(1-\alpha)R}^{R} S(n,r) dr} {\int_{0}^{R} S(n,r) dr} \\
              & = \frac{S(n,1)\int_{(1-\alpha)R}^{R} r^{n-1} dr} {S(n,1) \int_{0}^{R} r^{n-1} dr} \\
							& = \frac{R^n - ((1-\alpha)R)^n} {R^n}\\
							& = 1 - \left( 1 - \alpha \right)^n
\end{split}
\end{equation*}

If $x$ is at distance $r$ from the center, then $R-r$ is its distance to $y$,

\begin{equation*}
\begin{split}
\mathop{\mathbb{E}}_{x \in B} \left( \left\Vert P(x) \right\Vert \right) & = \frac{\int_{0}^{R} \left(R - r\right) S(n,r) dr} {\int_{0}^{R} S(n,r) dr} \\
              & = \frac{S(n,1)\int_{0}^{R} \left( R - r \right) r^{n-1} dr} {S(n,1) \int_{0}^{R} r^{n-1} dr} \\
							& = \frac{\frac{R^{n+1}}{n} - \frac{R^{n+1}}{n+1}} {\frac{R^n}{n}}\\
							& = \frac{R}{n+1}
\end{split}
\end{equation*}

\end{proof}
We have straightforward corollary.
\begin{corollary}
For any $\beta > 0$,
\[\mathop{\rm{Prob}}_{x \in B} \left\{ \left\Vert P(x) \right\Vert \leq \beta \right\} = 1 - \left(1 - \frac{\beta}{R} \right)^n\]
Furthermore,
\[\mathop{\mathbb{E}}_{x \in B} \left( \frac{\left\Vert P(x) \right\Vert}{R} \right) = \frac{1}{n+1}\]
\end{corollary}

\begin{figure}[t]
    \centering
		\includegraphics[width=\textwidth]{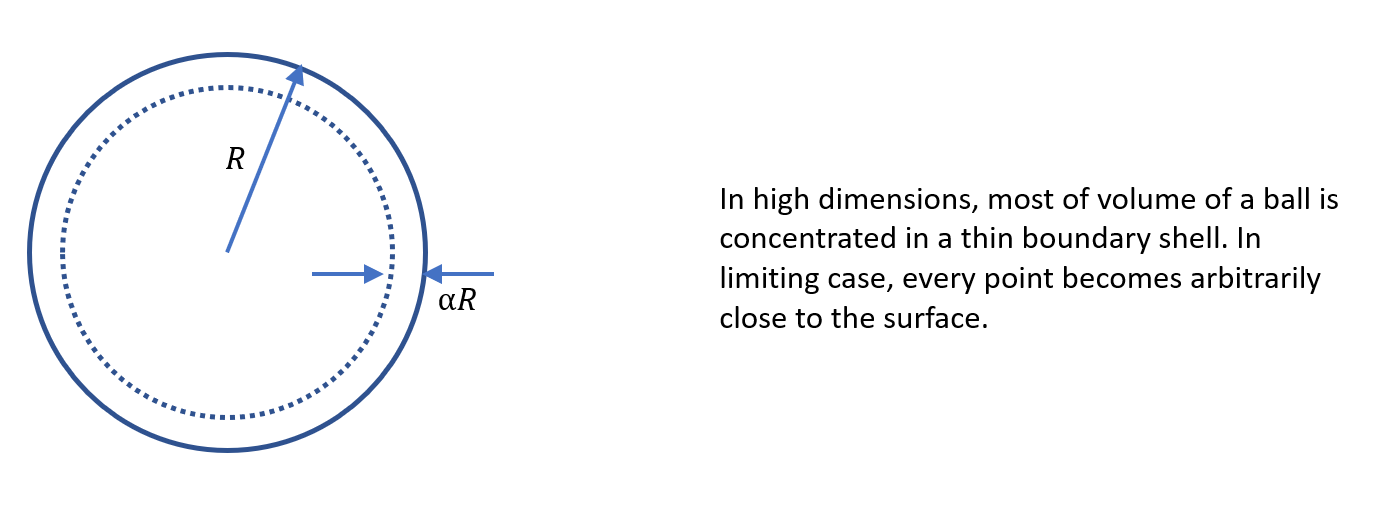}
    \caption{Proof of Lemma~\ref{lemma_sphere} depends on the fact that as $n$ increases, $n$-ball of any radius will have more fraction of volume closer to the surface. Therefore, probability that a point will be within a particular distance from the surface will increase for fixed radius. Since we are working with image manifolds, the radius of $n$-ball will also increase with $n$. Despite that, for 8-bit images $\in [0,255]^n$, the absolute distance to the surface shrinks to 0 as $n$ approaches $\infty$. Later in Section~\ref{section_statistics_natural} we will show a stronger result that for natural images $\in \mathbb{R}^n$ this holds true.}
    \label{fig_lemma_ball}
\end{figure}

See Figure~\ref{fig_lemma_ball}. Note that we can not conclude yet that expectation of $\left\Vert P(x) \right\Vert$ tends to 0 as $n$ approaches $\infty$ since we don't know how radius $R$ of image manifold will increase with $n$.  There are two cases.
\begin{enumerate}
\item
For 8-bit images, $B \subset [0,255]^n$ and is surrounded by a hypercube of diameter $255 {\sqrt n}$. Thus, $R=O(\sqrt{n})$, and $\mathbb{E}(\left\Vert P(x) \right\Vert)$ shrinks to 0.
\item
For idealized images with arbitrary real pixel values, $B \subset \mathbb{R}^n$. See Section~\ref{section_statistics_natural}  that $\mathbb{E}(\left\Vert P(x) \right\Vert)$ does approach 0 for natural image manifolds.
\end{enumerate}

\subsection{Image Manifolds with Arbitrary Geometries}
Now we present the main result on adversarial examples. We first define concept of radius of a manifold.

{\bf Definition:}
Let $M \subset \mathbb{R}^n$ be an $n$-dimensional bounded manifold with finite volume. Then, its {\em radius\/} is
\[{\rm Radius}(M) = R(M) = \frac{1}{2} \mathop{{\rm max}}_{x,y \in M} \left\Vert x-y \right\Vert\]

The diameter of the manifold, which is maximum pairwise distance, can be viewed as {\em dynamic range\/} of the corresponding image class. We want to bound the perturbation image relative to the radius. See Figure~\ref{fig_theorem_arbitray_manifold}.

\begin{theorem}\label{theorem_adv_negatives}
Consider any image classification problem at infinite resolution. Consider any solution of this problem by a deep learning image classification system $D_n = (A_n,T_n,G_n,Q_n)$ for any finite resolution $n = r \times c$. For any random sample $x \in \mathbb{R}^n$ which is classified positively for one of the image classes $c_x$, let $y \in \mathbb{R}^n$ be the closest adversarial example where $A_n(x) \neq A_n(y)$ and let $y = x + p$ where $p \in \mathbb{R}^n$ is the perturbation image. Denote the perturbation $L_2$ norm $\left\Vert p \right\Vert$ for $x$ by $P(x)$. Denote the trained manifold for image class $c_x$, for which $x$ was positively classified, by $M(x)$.

Assume that all trained manifolds in $Q_n$ are $n$-dimensional sets. Let $\alpha$, where $0 < \alpha \leq 1$ be the perturbation bound.

Then, over all $x \in Q_n$,
\[ \lim_{n \to \infty} \mathop{\rm{Prob}}_{x \in Q_n} \left\{ \left\Vert P(x) \right\Vert \leq \alpha \; R(M(x)) \right\} = 1\]
and
\[ \lim_{n \to \infty} \mathop{\mathbb{E}}_{x \in Q_n} \left( \frac{\left\Vert P(x) \right\Vert}{R(M(x))} \right) = 0\]
\end{theorem}

\begin{proof}
This is general case when the manifolds are arbitrary. For a given finite volume, $n$-ball minimizes the surface area and therefore maximizes the average distance of a point to the surface, which follows from the isoperimetric inequality~\cite{federer_isoperimetric_inequality}. The $n$-ball is the worst case manifold for the proof and if manifolds are all $n$-balls, then the theorem is proved by Lemma~\ref{lemma_sphere}. For any other geometry, the probability of being close to surface of the object will be strictly higher than the case for $n$-ball and the average distance to the surface will be smaller.

Without loss of generality consider the case when there is a single positive image class and a negative image class. Let the trained manifold be $M$ in $\mathbb{R}^n$. Let $B$ be $n$-ball such that
\[{\rm Volume}(M) = {\rm Volume}(B)\]
Then,
\[R(M) \geq R(B)\]
and for all $\beta > 0$,
\[\mathop{{\rm Prob}}_{x \in M} \left\{ \left\Vert P(x) \right\Vert \leq \beta \right\} \geq \mathop{{\rm Prob}}_{x \in B} \left\{ \left\Vert P(x) \right\Vert \leq \beta \right\}\]
and
\[\mathop{\mathbb{E}}_{x \in M} \left(\left\Vert P(x) \right\Vert \right) \leq \mathop{\mathbb{E}}_{x \in B} \left(\left\Vert P(x) \right\Vert \right) \]
Therefore, for $\alpha_B = \alpha \frac{R(M)}{R(B)} \geq \alpha$,
\begin{equation*}
\begin{split}
\mathop{{\rm Prob}}_{x \in M} \left\{ \left\Vert P(x) \right\Vert \leq \alpha \; R(M) \right\}
& = \mathop{{\rm Prob}}_{x \in M} \left\{ \left\Vert P(x) \right\Vert \leq \alpha_B \; R(B) \right\}  \\
& \geq \mathop{{\rm Prob}}_{x \in B} \left\{ \left\Vert P(x) \right\Vert \leq \alpha_B \; R(B) \right\} \\
& = 1 - (1 - \alpha_B)^n \\
& \geq 1 - (1 - \alpha)^n
\end{split}
\end{equation*}
Therefore limit follows for manifolds with arbitrary geomteries.

Proof of second part for expectation follows using similar steps.
\end{proof}

\begin{figure}[t]
    \centering
		\includegraphics[width=\textwidth]{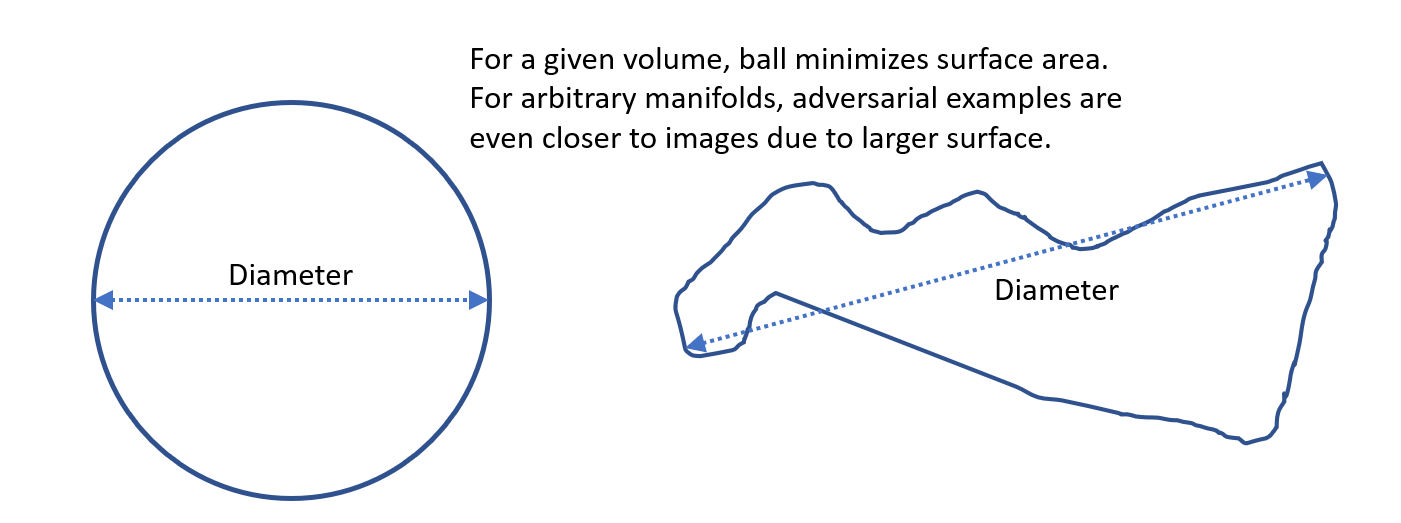}
    \caption{Proof of Theorem~\ref{theorem_adv_negatives} depends on the isoperimetric inequality according to which the $n$-ball attains minimum surface area among all possible $n$-dimensional objects with the same volume.}
    \label{fig_theorem_arbitray_manifold}
\end{figure}

For $k$-bit images, $M \subset [0,2^k-1]^n$ and is surrounded by a hypercube of diameter of the order of ${\sqrt n}$. Thus, $R(M)=O(\sqrt{n})$. Therefore we have the following theorem.
\begin{theorem}
\label{theorem_finite_pixel_range}
Let the pixel values be in a finite bounded range. Then,
\[\mathop{\mathbb{E}}_{x} \left( \left\Vert P(x) \right\Vert \right) = O\left(\frac{1}{\sqrt n}\right)\]
\end{theorem}

\begin{figure}[t]
    \centering
		\includegraphics[width=\textwidth]{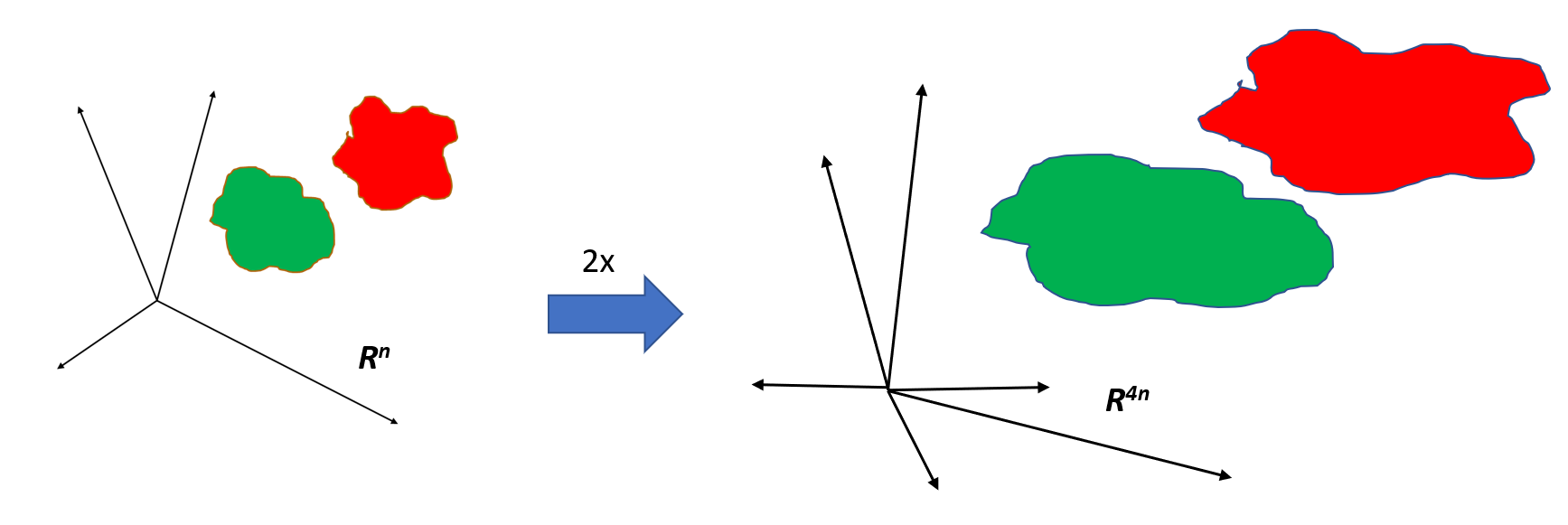}
    \caption{If image resolution is doubled, the radii of the manifolds increase. Proof of Theorem~\ref{theorem_finite_pixel_range} depends on the assumption that pixel values are $k$-bit, and therefore $R(M)=O(\sqrt{n})$. In Theorems~\ref{theorem_1f2} and~\ref{theorem_adv_neg_2}, we will remove this assumption for natural images.}
    \label{fig_manifold_expansion}
\end{figure}

Note that for the proof to work, we implicitly require that each positive image manifold is surrounded completely by its complement which will be the case if it is compact and the space is $\mathbb{R}^n$. For 8-bit images the total space in $[0,255]^n$ and as long as the positive manifolds are surrounded by their respective complements so that one can move out of the surfaces into the complement regions, proof will hold.

\subsection{Case of Linear Models}
In Section 3 of the reference~\cite{explaining_goodfellow}, it is argued that even simple linear models can suffer from adversarial negatives and problem in this argument has been independently highlighted earlier in~\cite{tanay_griffin}. The authors consider a simple linear model $w^T x$. Let $y = x + p$ be the adversarial negative with $p$ being the perturbation. See Figure~\ref{fig_google_fallacy}.

\begin{figure}[t]
    \centering
		\includegraphics[width=\textwidth]{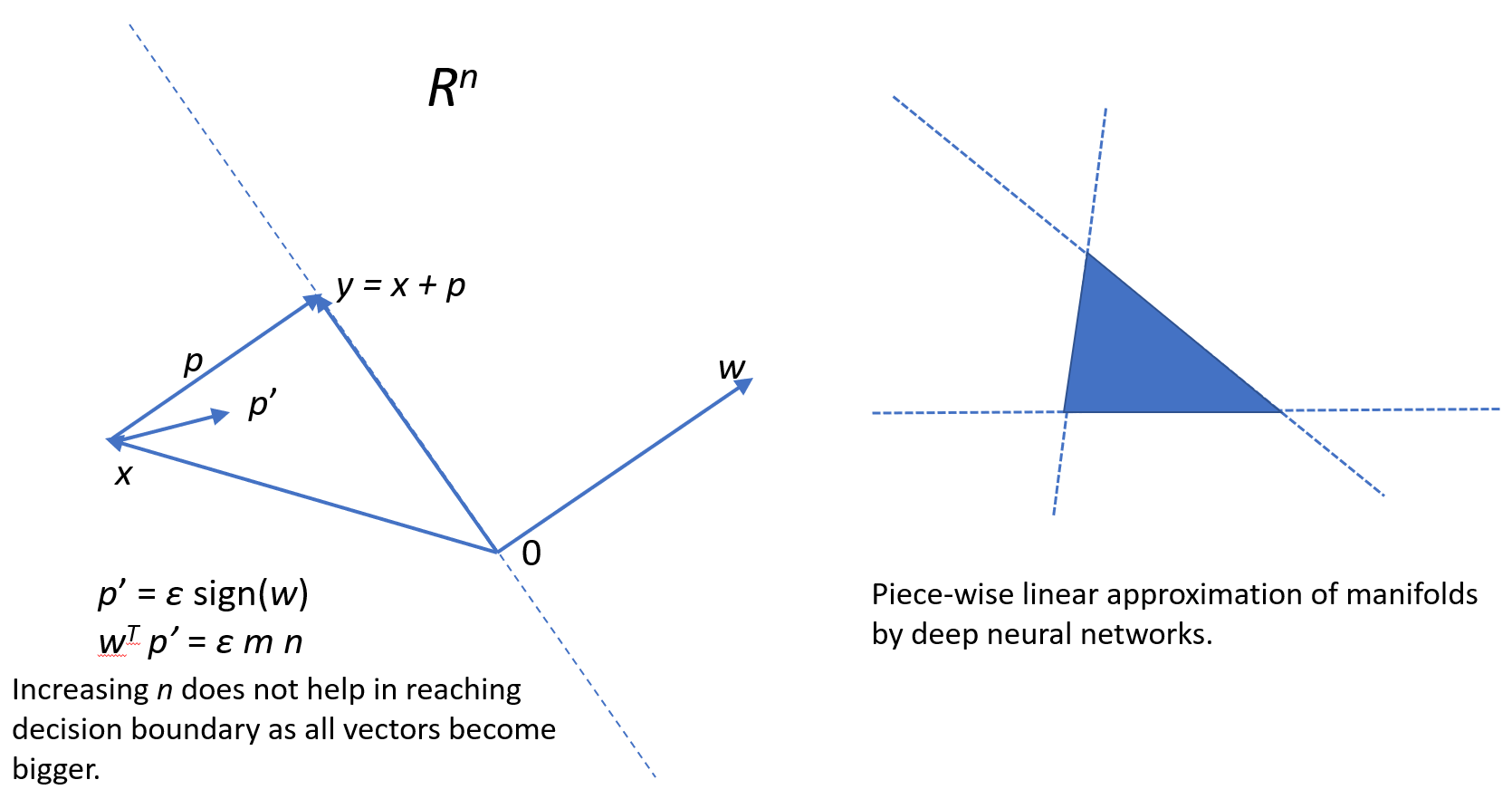}
    \caption{{\em Left}, mistake in a published paper by Goodfellow et al. in 2015 regarding an argument made to show that linear model $w^T x$ will suffer from adversarial examples. Fallacy lies in the fact that distances in spaces with different dimensions can not be compared, as everything becomes bigger as $n$ increases. {\em Right}, for deep networks we want results to apply to manifolds bounded by hyperplanes.}
    \label{fig_google_fallacy}
\end{figure}

Then,
\[w^T y = w^T x + w^T p\]
In~\cite{explaining_goodfellow}, $L_\infty$-norm of perturbation is constrained to a small value $\epsilon$ which is fixed to be the smallest greyness granularity, which is 1/255 of the dynamic range for 8-bit images. It is first argued that if $m$ is the average magnitude of coefficients of the weight vector $w$ and we choose $p' = \epsilon \; {\rm sign}(w)$ to be closely aligned with $w$, then for $n$ dimensions, $w^T p' = \epsilon m n$ will increase with $n$ to reach any target which will push the point on the other side of the decision hyperplane. We want $\epsilon m n$ to reach activation level to be negative of $w^T x$. It is argued in~\cite{explaining_goodfellow} that this is possible with increasing $n$. With increasing $n$, $w^T x$ also increases because $w$ and $x$ increase in their magnitudes as they have more elements in higher dimensional spaces. Therefore if we keep $L_\infty$-norm of perturbation constrained as in~\cite{explaining_goodfellow} to some small $\epsilon$, we can not generate adversarial examples using this argument for the linear model. In order to generate adversarial examples, we will have to set $L_\infty$-norm constraint to a higher value which will depend on the distance of $x$ from the decision boundary. But then there is no upper bound on the max norm due to fact that there is no upper bound on the distance of feature vector to the decision hyperplane. Even if we constrain the feature value range to be finite, such as [0,255] for 8-bit images, even then the distances can be fairly large in general case. Therefore, the generalization of the argument in~\cite{explaining_goodfellow} to deep networks is invalid.

Therefore, we can state the following.
\begin{theorem}
\label{linearity_theorem}
The linear model does not suffer from adversarial examples under $L_p$-norm constraint for any $p \geq 1$.
\end{theorem}
\begin{proof}
Proof follows from the fact that for any resolution $n$,
\[ \{ \left\Vert P(x) \right\Vert_p | x \in \mathbb{R}^n \} = \mathbb{R}^+ \]
where $\mathbb{R}^+$ is set of non-negative reals.
\end{proof}

At the same time, Theorem~\ref{theorem_adv_negatives} in this paper will hold for all bounded manifolds. This is because image manifolds are assumed to be bounded and surrounded by their complement. As $n$ increases, the surface increases rapidly having more volume close to it, thereby bringing points closer to some surface where adversarial examples exist. The size of the manifold also increases and therefore the perturbation bound in Theorem~\ref{theorem_adv_negatives} is expressed relative to this size. The size of the manifold does not increase that fast, and therefore perturbations become smaller even in absolute sense. This will also hold if the manifold is bounded from all sides by hyperplanes as a special case as shown in Figure~\ref{fig_google_fallacy}. Deep networks perform this piecewise linear approximation of functions with ReLU activation, see~\cite{bengio_linear}, so our results apply to deep networks as a special case.

\section{Stochastic Gradient Descent and Optimization Landscape}
\label{section_sgd}
Let's consider the question why the deep learning works well. Why has there been success in training deep networks? One of the reasons is that it does not get stuck in local minima, see~\cite{yann_dauphin}. Neural networks of smaller sizes get trapped in local minima but this seems to be unlikely in the case of deep networks. It has been empirically shown that most of critical points encountered during stochastic gradient descent are saddle points as they are more numerous than local minima, see~\cite{loss_surface_choromanska, yann_dauphin} for these results and how random matrix theory can be applied to investigate this.

\subsection{Optimization Landscape Polynomials}
We adopt a related but different approach and use results from theory of random polynomials which give explicit bounds on number of critical points. In order to understand this well we will bound the number of saddle points and that of local minima by approximating loss surfaces, also referred to as {\em optimization landscapes\/}, with polynomials. For polynomial approximation of deep networks, see~\cite{poggio_2, poggio_3} where ReLU is approximated by a polynomial and results on singularity of the Hessian are derived for overparameterized systems.

Key observation is that all the operations used in deep learning are continuous even though they may not be differentiable.

\begin{theorem}
\label{theorem_opti_poly}
Optimization landscape of a deep learning network can be approximated by a polynomial.
\end{theorem}
\begin{proof}
The proof follows from Stone-Weierstrass approximation theorem which states that any continuous function on a compact set can be approximated by a polynomial~\cite{stone}.
\end{proof}

The continuity of functions is obvious for convolution layers, fully connected layers and differentiable loss functions. This is true of nonlinear layers such as ReLU and Max Pooling. Though they are not differentiable but they are continuous as their graphs don't have breaks for continuous inputs. ReLU is piecewise linear. To see immediately how ReLU activation layer can be approximated by a polynomial, notice that it can be written as follows
\[
\rm{ReLU}(x)  = \rm{max}(0,x) = \frac{1}{2} (x+|x|)
\]
in terms of absolute function which is known to have polynomial approximation. Max pooling is continuous for continuous inputs. Smooth versions of max are well-known. Notice that max can be expressed as follows
\begin{equation*}
\rm{max}(x,y)  = \rm{ReLU}(x-y) + \frac{\rm{ReLU}(x-y)}{|x-y|} y + \rm{ReLU}(y-x) + \frac{\rm{ReLU}(y-x)}{|x-y|} x
\end{equation*} 

In fact, we have two kinds of polynomials:
\begin{itemize}
\item
{\em Optimization Landscape Polynomials}. We fix the image pixel values. Neural networks parameters are variables.
\item
{\em Image Polynomials}. We fix the neural network parameters. Image pixel values are variables.
\end{itemize}

In both cases, we will have a very large polynomial. The degree of the polynomial goes up as layers in the network increase. As number of network parameters and image resolution go up, then the number of variables goes up.

In this section, we are concerned with optimization landscape polynomials. In Section~\ref{section_statistics_natural}, we will discuss image polynomials.

What can we say about number of saddle points and that of local minima of deep network polynomials? Since deep learning is a practical field, one will have to make assumptions based on empirically derived statistics on the coefficients of the polynomials based on particular applications. One will have to endow the space of deep network polynomials with a probability measure in order to derive mathematical results. We make use of results from theory of random polynomials endowed with Gaussian probability measure with certain assumptions on variances, see~\cite{dedieu_malajovich, dean_majumdar}, which provides evidence why deep learning works well in practice.

\begin{theorem}[\emph{Critical points of random polynomials, see~\cite{dedieu_malajovich}}]
\label{theorem_dedieu}
Let $C_{d,n}$ denote the expected number of critical points of a random polynomial of degree at most $d$ in $n$ variables,  and $E_{d,n}$ the expected number of minima. Let $P_n$ be probability that a critical point is a local minima. Then,
\begin{equation*}
\begin{split}
C_{d,n} & \leq \sqrt{2} (d-1)^{(n+1)/2} \\
E_{d,n} & \leq C_{d_n} P_n 
\end{split}
\end{equation*}
and
\[P_n \leq C \ \rm{exp} \left( -n^2 \frac{\ln 3}{4} \right) \]
for some positive constant $C$.
\end{theorem}
\begin{corollary}
For some positive constant $K$,
\[E_{d,n} \leq K \ \rm{exp} \left( -n^2 \frac{\ln 3}{4} + \frac{n+1}{2} {\ln(d-1)} \right) \]
\end{corollary}

For proof, see~\cite{dedieu_malajovich, dean_majumdar}. This result shows that most large random polynomials have only saddle points and local minima become increasingly rare. Assuming this result holds for most practical problems solved by deep learning, this implies that the number of local minima becomes arbitrarily small compared to number of saddle points as the resolution of images and size and depth of deep networks increase. The probability that all eigen values of the Hessian matrix of these polynomials are positive falls off very rapidly~\cite{dean_majumdar}.

Under the assumption that the loss surfaces of deep neural networks are normal random polynomials as in~\cite{dedieu_malajovich, dean_majumdar}, we can state the following.
\begin{prop}
Let $L(n)$ denote the ratio of the expected number of minima to the expected number of critical points of optimization landscape of a deep neural network of size $n$. Then,
\[\mathop{\lim}_{n\to\infty}L(n) = 0\]
\end{prop}

\subsection{Optimization on Loss Surfaces}
It has been shown in~\cite{bray_dean, loss_surface_choromanska, yann_dauphin}, both theoretically under certain assumptions as well as empirically, that the local minima are concentrated towards the bottom of the optimization landscape. Therefore during SGD we are likely to encounter local minima only towards the end of training which correspond to good enough solutions to practical problems. Earlier in the training, it is quite likely that slowing down is occuring because of a saddle point.

The above results also indicate why SGD and various tricks and techniques work in deep learning. Assume for a particular minibatch of training images, the deep network is at local minima or saddle point. As new minibatches arrive with their varying image statistics, optimization landscape changes and it becomes less likely that we will continue to be at the same critical point and it is more likely that we will escape from it despite slow down in the training time. This provides justification for stochastic algorithms.

One can also ask why empirically tested techniques such as ReLU activation, batch normalization and dropout have been effective in making training faster and more robust. Consider ReLU in which a dead neuron can become alive or vice versa thereby having ripple effects on the neurons it is connected with in higher layers and altering the optimization landscape significantly. This would probabilistically help SGD in escaping from a critical point. Let {\em index\/} be the fraction of negative eigen values of the Hessian matrix. Any empirically discovered technique which results in more than 50\% chance of reduction in index will help the training. Higher the chance, more effective it will be. These techniques perturb the optimization landscape in a statistically beneficial way and therefore they are helpful in making training faster with better solutions.

\begin{prop}
Any technique which speeds up escape from a critical point will speed up the deep learning training time. Furthermore, any technique which speeds up reduction of the Hessian index will speed up the deep learning training time.
\end{prop}

\section{Statistics of Natural Images, Adversarial Examples and Manifold Learning Hypothesis}
\label{section_statistics_natural}
Natural images have their own particular characteristics. We now derive stronger results for them.

\subsection{Perturbation Adversarial Examples}
It has been know that natural signals follow 1/$f^\beta$ process. For natural images which are 2-D signals $\beta$ is around 2~\cite{1f_images}. The power in different 2-D frequencies in natural images is inversely proportional to square of frequencies. Let us formulate this in terms of discrete wavelet multiresolution representation of natural images. A 2-D image consists of low frequency version of the image combined with high frequencies. By adding higher frequencies, the image resolution is doubled. Wavelets decompose the image into frequency subbands. LL subband corresponds to low frequencies and LH, HL and HH subbands to high frequencies of different orientations, see~\cite{wavelets_strang}. See Figure~\ref{fig_wavelet}. LH frequencies are cross-product of 1-D low frequencies in vertical direction and 1-D high frequencies in horizontal direction. HL are other way round. HH are high frequencies in both directions. Therefore the norm of the image increases as resolution is doubled. It increases primarily because there are more pixels in the image. Besides due to upsampling of existing low frequencies, it increases due to the addition of higher frequencies in the next octave but this increase follows a decreasing geometric sequence as per 1/$f$ process for LH and HL frequencies and 1/$f^2$ process for HH frequencies.

\begin{figure}[b]
    \centering
		\includegraphics[width=\textwidth]{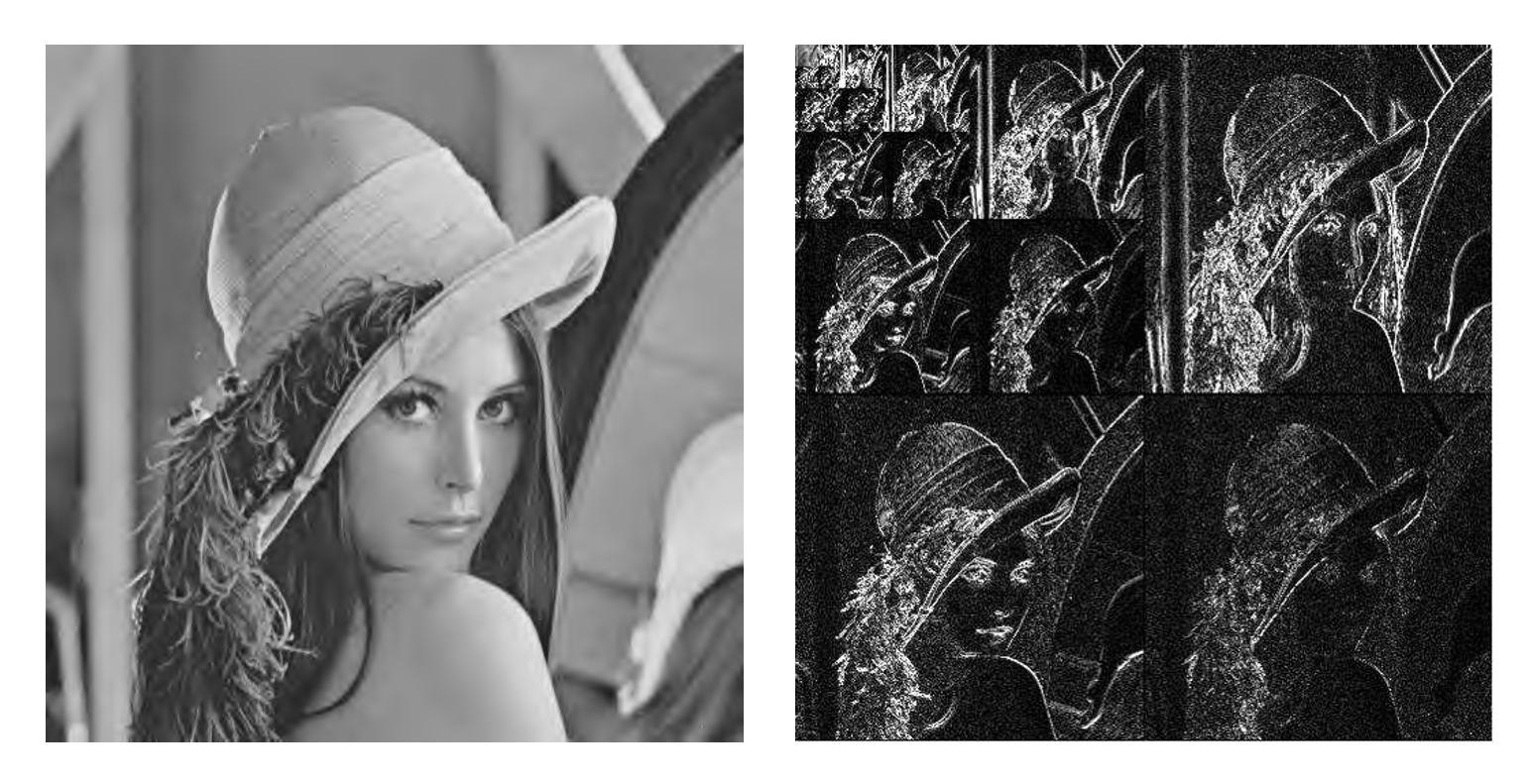}
    \caption{Multi-resolution orthonormal wavelet representation of an image.}
    \label{fig_wavelet}
\end{figure}

This allows us to compute the bound on radius of manifolds as the resolution increases. Underlying assumption is that we are working with high accuracy deep neural networks and trained and ground truth manifolds have approximately same radii, and in fact, can be even of the same order of magnitude for proof to work fine.

\begin{theorem}
\label{theorem_1f2}
For a semantic class of natural images, let $M_n$ be ground truth image manifold in space $\mathbb{R}^n$ for finite resolution $n = 2^m \times 2^m$, where $m \geq 0$, and let the images at different resolutions follow multi-resolution orthonormal wavelet representation which obeys 1/$f$ and 1/$f^2$ power spectrum processes. Let $M_n'$ be the trained manifold and assume ${\rm Radius}(M_n) \approx {\rm Radius}(M_n')$. Then,
\[{\rm Radius}(M_n) \thicksim {\sqrt n}\]
\[{\rm Radius}(M_n') \thicksim {\sqrt n}\]
\end{theorem}
\begin{proof}
Consider a $2^k \times 2^k$ image with power (energy per pixel) $l_{\rm LL}$ in its low frequencies and $h_{\rm LH}$,  $h_{\rm HL}$ and $h_{\rm HH}$ in high frequencies. Consider its next higher resolution $2^{k+1} \times 2^{k+1}$. Power will be now $l_{\rm LL}$ in low frequencies and $h_{\rm LH}+h_{\rm LH}/2$, $h_{\rm HL}+h_{\rm HL}/2$ and $h_{\rm HH} + h_{\rm HH}/4$ in high frequencies. In limit, we will have
\[l_{\rm LL} + 2 h_{\rm LH} + 2 h_{\rm HL} + \frac{4}{3} h_{\rm HH}\]
Therefore energy will be bounded by
\[ n \; \left( l_{\rm LL} + 2 h_{\rm LH} + 2 h_{\rm HL} + \frac{4}{3} h_{\rm HH} \right)\]
Orthonormal property ensures that energies in image space and in frequency space are same. Therefore $L_2$-norm of any image will be bounded
\[ C {\sqrt n} \]
for some constant $C > 0$.
Therefore,
\[R(M_n) \thicksim {\sqrt n}\]
and since radius of $M_n'$ is approximately very close to that of $M_n$,
\[R(M_n') \thicksim {\sqrt n}\]
\end{proof}

For natural images, we can now improve the results on perturbation adversarial examples which we derived in Theorem~\ref{theorem_adv_negatives}.
\begin{theorem}
\label{theorem_adv_neg_2}
Consider the statements of Theorem~\ref{theorem_adv_negatives} and Theorem~\ref{theorem_1f2}. Then,
\[\lim_{n \to \infty} \mathop{\mathbb{E}}_{x} \left( \left\Vert P(x) \right\Vert \right) = 0.\]
\end{theorem}
\begin{proof}
The proof is obvious for $n$-balls using Lemma~\ref{lemma_sphere} and Theorem~\ref{theorem_1f2}. For manifolds with arbitrary geometries, $n$-balls are worst case scenario using the same arguments as in Theorem~\ref{theorem_adv_negatives}. For arbitrary manifolds more volume is concentrated near the surface area compared to $n$-balls with same volumes as per the isoperimetric inequality, see~\cite{federer_isoperimetric_inequality}.
\end{proof}

\subsection{Unrecognizable Adversarial Examples}
Besides power spectrum properties we can apply the {\em Manifold Learning Hypothesis\/} to understand geometry of image manifolds. The Manifold Hypothesis states that most natural image classes at large enough resolution $n  = r \times c$ form manifolds which are embedded in a topological subspace with dimensionality $f \ll n$, see~\cite{manifold_ams_paper, manifold_book}. We can consider this subspace with much lower dimensionality as pose space where each point corresponds to a pose of the image as determined by some pose parameters. See Figure~\ref{fig_pose_space}.

\begin{figure}[t]
    \centering
		\includegraphics[width=\textwidth]{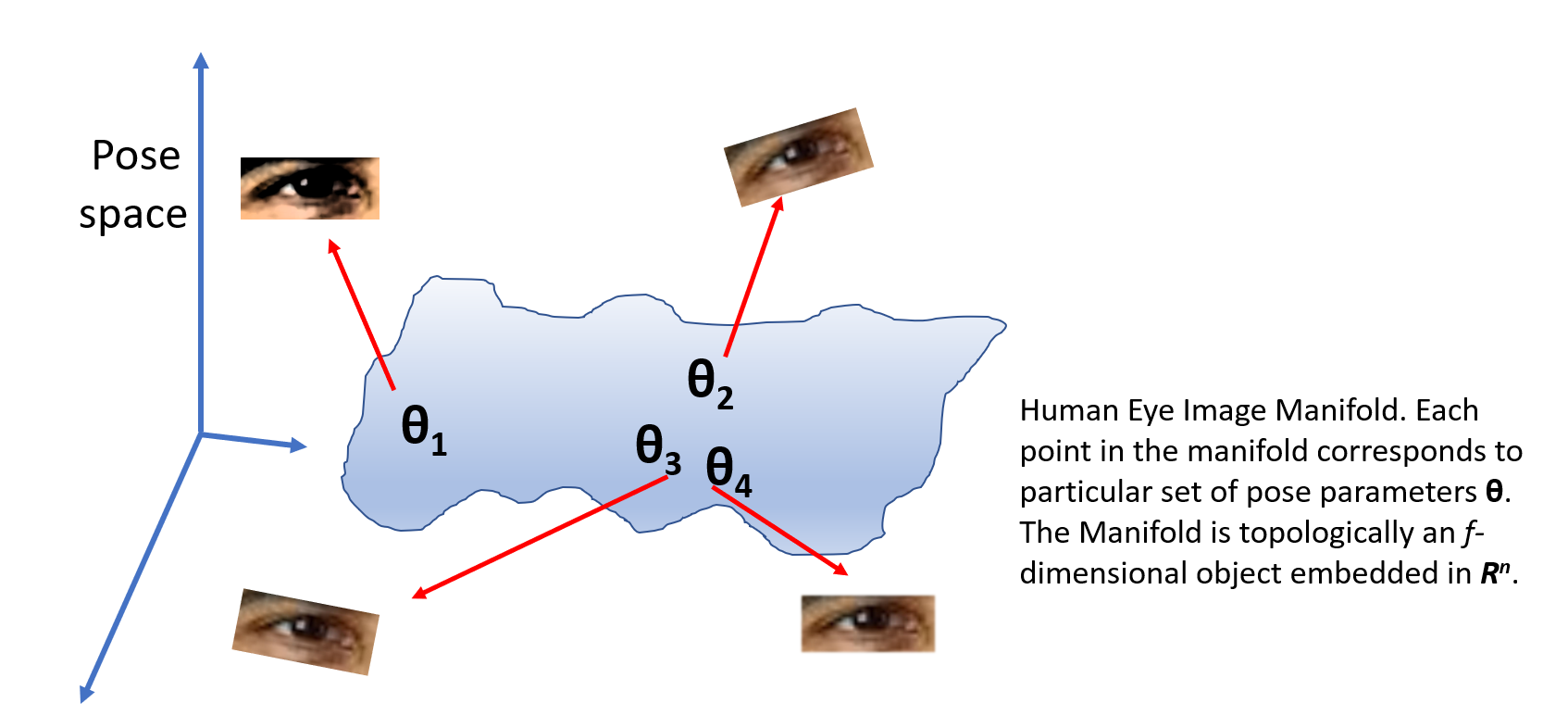}
    \caption{The Manifold Learning Hypothesis.}
    \label{fig_pose_space}
\end{figure}

In addition to adversarial examples discussed in Theorems~\ref{theorem_adv_negatives} and~\ref{theorem_adv_neg_2}, it has been shown that it is easy to generate artificial images, some of which could visually look like random noise and are unrecognizable by humans, for which the deep network returns a positive class with high probability. See Figure~\ref{fig_adv_examples}. One starts with some completely random image, which could be just pure noise, and performs gradient ascent on image pixels to maximize the output probability for some class. Soon the algorithm converges to a fake unrecognizable image with high enough output probability.

This is easy to understand if there is no negative class and there are only positive classes. Discriminative loss based training does not care what happens to negative space then. If there is negative class, then we need to understand it better. This phenomenon can be then understood using Manifold Hypothesis and Theorem~\ref{theorem_dedieu}. Manifolds of positive image classes occupy very small volume in high dimensional spaces compared with the complement corresponding to the negative class, as per the Manifold Hypothesis.

\begin{prop}
Consider 8-bit grayscale images. Assume the Manifold Learning Hypothesis is true. Let image manifolds be $f$-dimensional topological objects in the $n$-dimensional space $[0,255]^n$.  Let manifolds have finite $f$-dimensional volume. Then as $n$ becomes arbitrarily large, the volume of positive image manifolds becomes arbitrarily small compared to that of the surrounding negative space.
\end{prop}

To prove the above proposition for a simple case, consider a resolution $n = r \times c$ for 8-bit images. Let $u_n$ be the number of all possible images in the universe $U = [0,255]^n$ and let $c_n$ be the number of images in a particular semantic class $C \subset U$. Volume computation reduces to counting number of images in discrete domain. Consider next resolution $4n = 2r \times 2c$. For each pixel of each image in $U$ at resolution $n$, we have 3 new pixels (in general, 3 degrees of freedom) as it is upsampled to $2\times 2$ region. There are $k = 256^3$ possible choices of values of these pixels. Assume that for each pixel for any image in $C$, there are $t < k$ choices for these 3 pixels due to semantic constraints which forces neighboring pixels to have strong correlation with each other in natural images.
Then,
\[ u_{4n} = u_n \; k^n\]
\[ c_{4n} = c_n \; t^n\]
and $u_n$ will become arbitrarily large compared with $c_n$ as $n$ increases. In fact, because of 1/$f^\beta$ process, $t$ will become smaller with $n$, though for the proof we just need $t < k$.

An object with finite $f$-dimensional measure will have zero $n$-dimensional measure if $f < n$. And as $n$ tends to $\infty$ it is not possible to have enough training data for the negative class due to this curse of dimensionality. In order to get enough training data for negative class, we will have to sample points from a volume just outside the surface of the manifold. This volume is $n$-dimensional space surrounding an $f$-dimensional object. Even if manifold was $n$-dimensional, since most of the volume is concentrated on the surface, this volume for negative samples will be very large, see Theorem~\ref{theorem_halo} in later section. Therefore,
\[{\rm Negative\; Training\; Data\; Size} \ggg {\rm\; Positive\; Training\; Data\; Size}\]

Consider the case when we are not able to completely surround the manifold by negative training samples, and there are some gaps. Training of deep networks does not guarantee anything in large spaces which are not covered by the training data. Therefore, the training may not choose to cover this gap under discriminative loss function depending upon the local geometry of the manifold. Fake unrecognizable adversarial examples will be then found in the negative space through this gap. See Figure~\ref{fig_neg_examples}. In the figure, we see why under discriminative loss, the deep network may not invest in neurons to close such gaps, thereby creating fake adversarial space.

\begin{figure}[t]
    \centering
		\includegraphics[width=\textwidth]{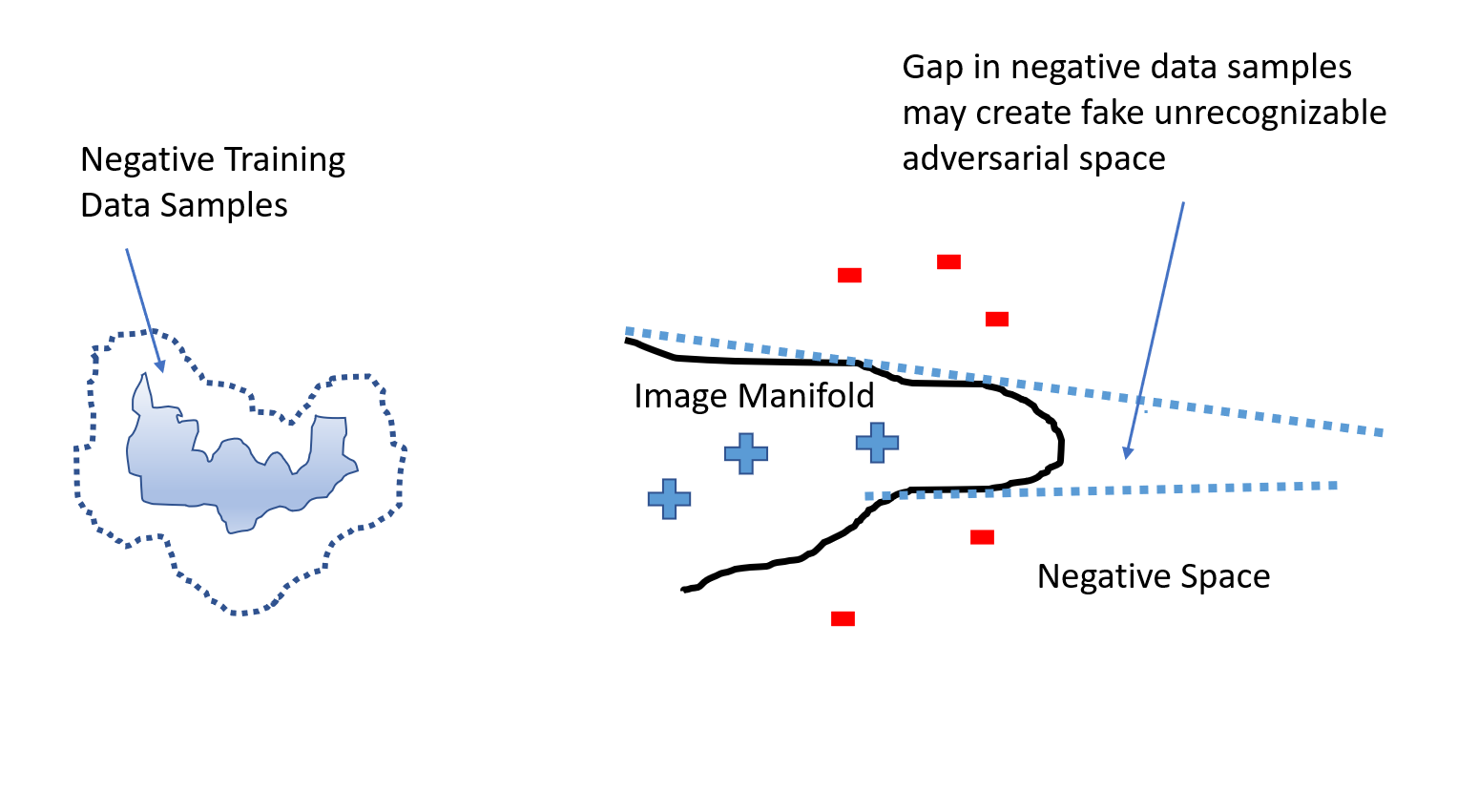}
    \caption{{\em Left}, image manifold is $f$-dimensional object with finite $f$-dimensional volume embedded in an $n$-dimensional space which implies that need for training data for negative samples surrounding the surface of the manifold becomes arbitrarily large as $n$ increases. {\em Right}, any gap in the surrounding negative training data may lead to fake adversarial images depending upon the geometry of the surface. In addition to the two hyperplanes shown, with more training data there should have been another one closing the gap.}
    \label{fig_neg_examples}
\end{figure}

At the same time, number of critical points is increasing as underlying Image Polynomials become large, see Theorem~\ref{theorem_dedieu}. Goal of training is to approximate the characteristic functions of the image manifolds (probability value 1 inside a image manifold and 0 elsewhere). The final Image Polynomials for different classes approximate these functions under discriminative loss optimization. Note that in image polynomials, variables are image pixels and the deep network parameters contribute to coefficients of the polynomials. Therefore it becomes easier to find these random looking adversarial samples from the negative space, which has much larger volume, using gradient ascent algorithm. See Figure~\ref{fig_fake_examples}.

\begin{prop}
Fake unrecognizable adversarial examples correspond to critical points of Image Polynomials in increasingly large negative space and which become more and more numerous with increasing resolution of images and increasing size of networks trained under discriminative loss function, as given by Theorem~\ref{theorem_dedieu} and Theorem~\ref{theorem_opti_poly}.
\end{prop}

\begin{figure}[b]
    \centering
		\includegraphics[width=\textwidth]{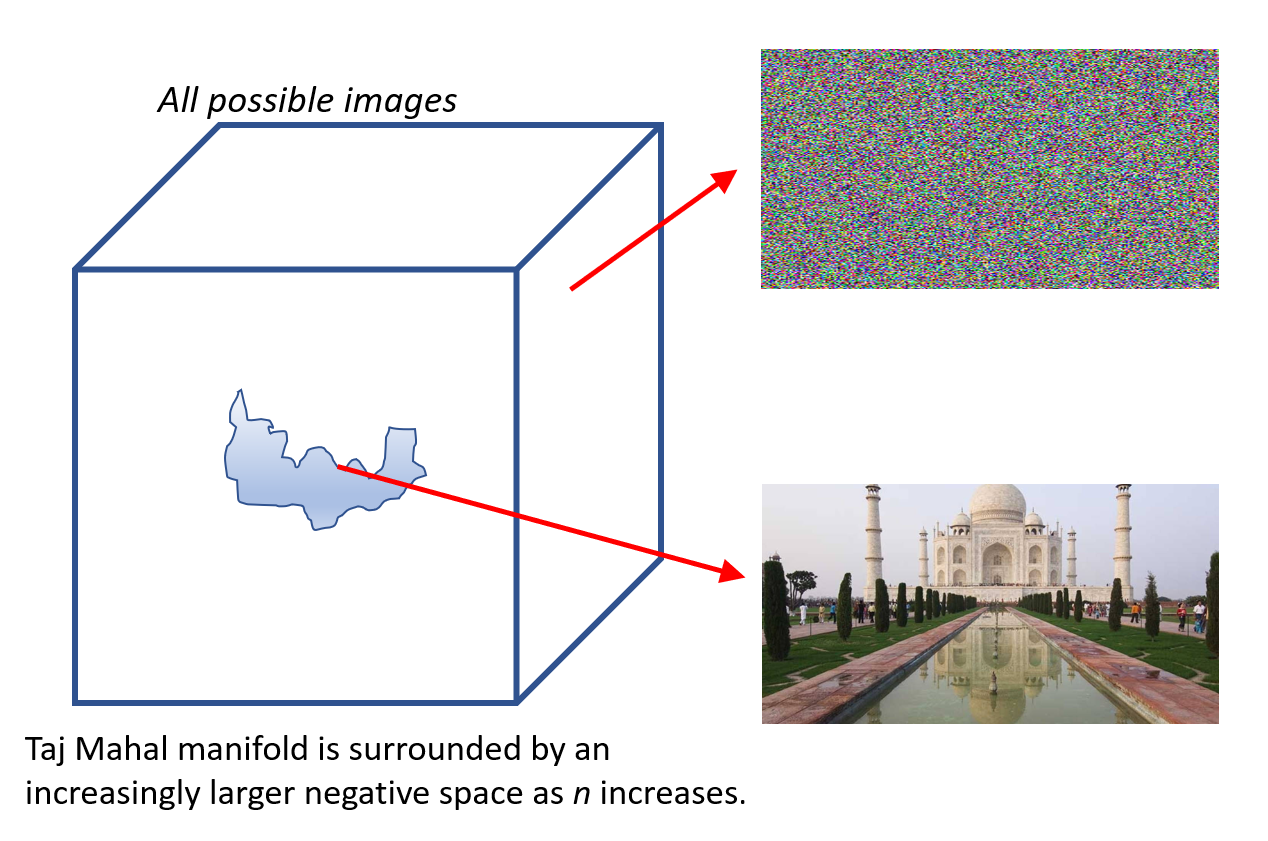}
    \caption{The Manifold Learning Hypothesis indicates that volume occupied by image manifolds is much smaller than surrounding complement negative space. The negative space abounds with critical points of Image Polynomials.}
    \label{fig_fake_examples}
\end{figure}

Another application of the Manifold Hypothesis is in providing the following justification for deep learning. In deep learning, feature engineering is done automatically. A deep neural network can be understood as a feature extractor followed by a classifier. The first part transforms the input of dimensionality $n$ to a feature of dimensionality $f$, which ideally should correspond to subspace dimensionality of the image manifold, and then this feature extractor is used for subsequent classification (or regression) tasks.

\section{Eliminating Adversarial Examples and Discussion}
\label{section_future}
How can we eliminate adversarial examples? Is it curse of dimensionality that image classes will always have border and every sample eventually happens to be close to the border in high dimensions? In this section, we dive into problem which also suggests way for future improvements of deep learning.

\subsection{Training Data}
A natural idea to mitigate the problem of adversarial examples is to include them in the training data. However, no matter how much you include these cases, you can not eliminate the problem due to Theorems~\ref{theorem_adv_negatives} and~\ref{theorem_adv_neg_2}.

As long as there is surface, the problem persists. One can question if the concept of hard borderline for ground truth manifolds is robust. Take an image of a cat and start modifying it. When does it stop being a cat? It could very well be subjective opinion. One can consider dilating the manifold by including all those images which are visually somewhat close to cat images and human judgment assigns them a probability less than 1 of being a cat. This creates a halo around the manifold. Hopefully this will mitigate the problem of adversarial examples when we perturb only those images for which ground truth probability is 1. If the original ground truth manifold is a subset of trained manifold obtained by training on the expanded manifold then there won't be any adversarial examples as per original ground truth.

Let's quantify the need for additional training data for this purpose, some of which can be created using {\em data augmentation\/} techniques. Define the {\em dilation\/} of a ground truth manifold $M$ by $n$-ball of radius $r$ to be the set
\[M \oplus B(0,r) =  \{x \in B(y,r) | y \in M\}\]
where $B(y,r)$ is $n$-ball centered at $y$ of radius $r$.

What should be the value of dilation $r$? That could be dependent on the image class and its radius $R$. We will consider two cases. The first case is optimistic and we use the bound $R/n$ from Lemma~\ref{lemma_sphere}. Second case is pessimistic where we estimate $r$ to be a fixed fraction of $R$.

\begin{theorem}
\label{theorem_halo}
Let $M$ be $n$-ball of radius $R$. Dilate $M$ by an $n$-ball of radius $r$. Let $\alpha > 0$. Then,
\begin{enumerate}
\item
If $r = \alpha \frac{R}{n}$, then
\[ \mathop{\lim}_{n\to\infty} \frac{{\rm Vol} (M \oplus B(0,r))} {{\rm Vol} (M)} = e^\alpha \]
\item
If $r = \alpha R$, then
\[ \mathop{\lim}_{n\to\infty} \frac{{\rm Vol} (M \oplus B(0,r))} {{\rm Vol} (M)} = \mathop{\lim}_{n\to\infty} (1 + \alpha)^n = \infty\]
\end{enumerate}
\end{theorem}
\begin{proof}
The proof follows from the computation of volumes using integration as in Lemma~\ref{lemma_sphere}.
For first part, evaluate
\[\frac{\int_{0}^{R + \alpha \frac{R}{n}} S(n,r') dr'} {\int_{0}^{R} S(n,r') dr'} \]
which is
\[\left(1 + \frac{\alpha}{n}\right)^n\]
\end{proof}

Thus in worst case, we may need much more training data which may be practically infeasible.

The concept of borderline halo makes us consider new definitions of what ground truth is and what test error is. We can define {\em dilated test error\/} of deep network to be 0 if
\[M \subseteq M' \subseteq M \oplus B(0,r)\]
where $M'$ is trained manifold. And by defining $M \oplus B(0,r)$ where $r = \alpha R$ as {\em dilated ground truth\/}, we can develop intuition behind the statement that almost everything is at surface. What is really true then is that almost everything is an augmented and borderline image.

The above theorem can be also applied to the case of fake unrecognizable adversarial images, as it quantifies the need for training data for negative samples which surrounds the manifold in a similar fashion.

In high dimensions, any practical size training data will be sparse. And with more diversity in images due to perturbations, we will need larger capacity networks. This limits the practical applicability of this solution which attempts to create a distance from the surface of the manifold. The simplistic brute force manifold dilation is a subjective reinterpretation of ground truth and the need for larger training data is independent of that. One can consider more sophisticated methods to get additional training data but one has to ask if the neural networks will be able to use it effectively. Training data always helps but the roots of the problem may be deeper than just lack of enough training data, see Figure~\ref{fig_data_aug}.

\begin{figure}[b]
    \centering
		\includegraphics[width=\textwidth]{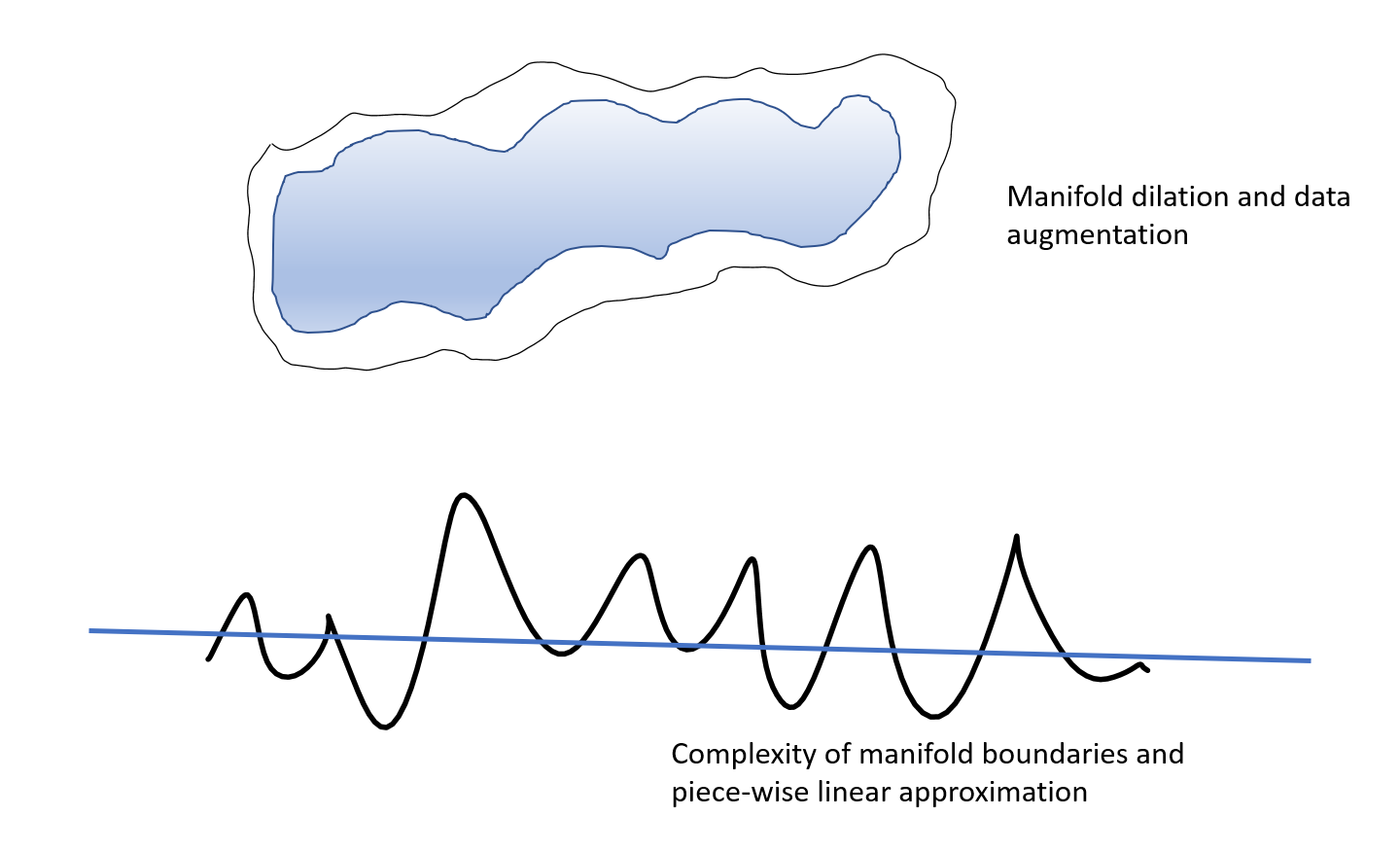}
    \caption{{\em Top}, creating a halo around the surface of ground truth image manifold through data augmentation. {\em Bottom}, the surface of manifolds may be too complex and highly non-linear as evidenced by Adversarial Examples which show that deep networks, despite their better performance over other machine learning techniques, suffer from having such images on the incorrect side of their decision hyperplanes.}
    \label{fig_data_aug}
\end{figure}

\subsection{Surfaces of Image Manifolds}
So far, assumption was made that we are interested in ever increasing dimensionality $n$ as in Theorems~\ref{theorem_adv_negatives} and~\ref{theorem_adv_neg_2}. But the Manifold Hypothesis indicates that for an image class, there exists a finite topological dimension $f$ of underlying manifold. So assumption of arbitrarily large $n$ is wrong and therefore in principle we should be able to eliminate adversarial examples, as we are not interested in the limit in theorems, but in the perturbation bound for finite case.

Even though $f$ may be finite, the image manifolds have complex geometries, see Figure~\ref{fig_data_aug}. Though $f \ll n$, the $f$-dimensional image manifold can still twist and turn around in the whole $n$-dimensional space, very much like how a fractal curve does. A fractal curve can turn around so much that it can result into an everywhere continuous curve which is nowhere differentiable or even fill up the entire space. We will build visual intuition behind the complexity of manifold surfaces.

In fact, objects with sharp boundaries belong to manifolds which are non-differentiable everywhere, see~\cite{wakin_manifold}.

To build further intuition behind the complexity of manifold surfaces, consider the multiresolution family of manifolds for an image class
$\{M_n | n = 4^k\}$
and how one obtains $M_{4n}$ from $M_{n}$ by image zooming. For an image $x$ in $M_n$, the corresponding images in $M_{4n}$ are obtained by following stochastic 1/$f^\beta$ process in frequency domain and self-similarity in space domain. If $x$ is close to the surface in $M_n$, we can expect that this upsampling will result in roughness in surface of $M_{4n}$. As we add details to blades of grass, fur of cat or outline of clouds, while obeying the characteristics of natural images, those properties will manifest in terms of mathematical properties of the surfaces of the manifolds. See Figure~\ref{fig_fractal_examples}. Note that here we are zooming into manifolds in a dimensional sense where the underlying dimension of spaces becomes bigger from $n$ to $4n$. This is different from the standard 2-D image zoom in which space remains $R^2$.

\begin{figure}[t]
    \centering
		\includegraphics[width=\textwidth]{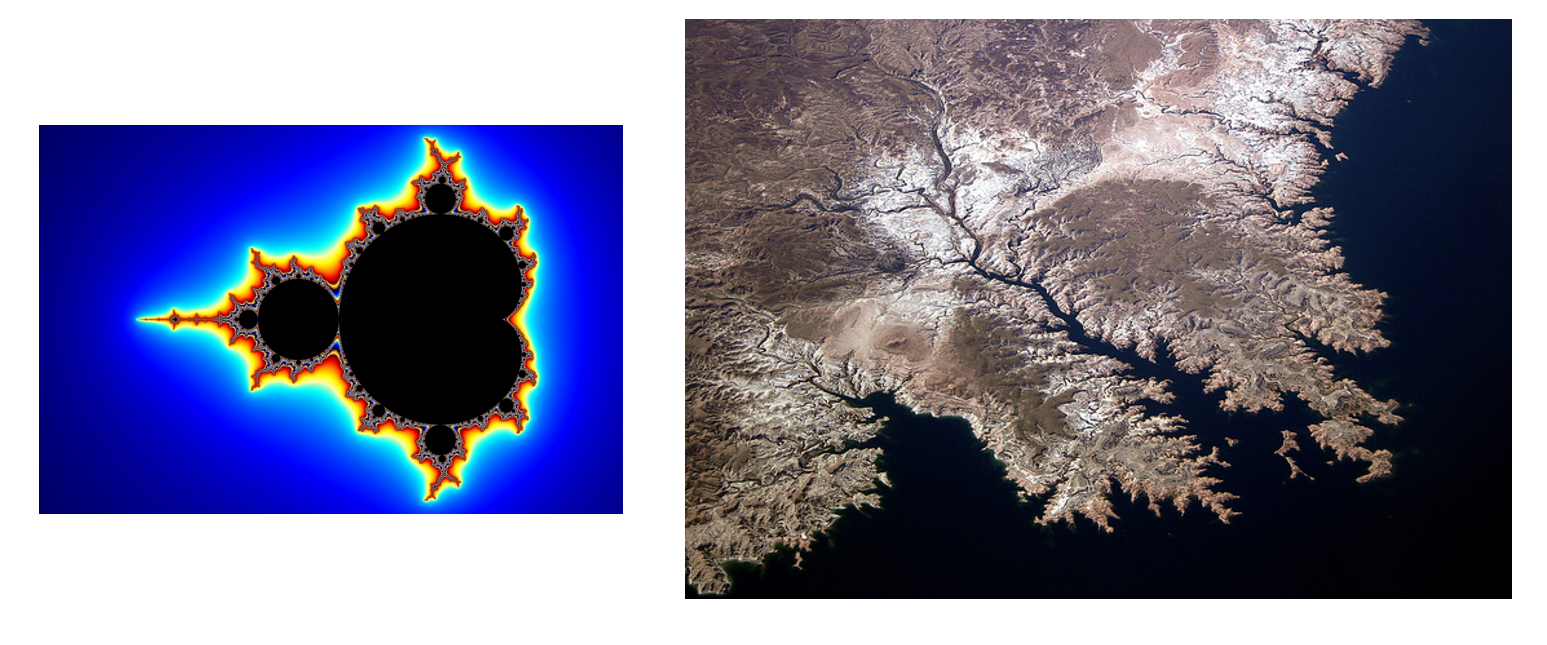}
    \caption{{\em Left}, Mandelbrot set with highly complex boundary. {\em Right}, self-similarity in natural images such as coastlines along with stochastic 1/$f^\beta$ process is likely to lead to rough surfaces of image manifolds.}
    \label{fig_fractal_examples}
\end{figure}

Finally, fix the dimension and make the object change its pose. If the object is complex with several degrees of freedom in its transformations, that will reveal in twists and turns of the manifold.

Now we will use some mathematical concepts to make the above intuition rigorous.

\begin{figure}[b]
    \centering
		\includegraphics[width=\textwidth]{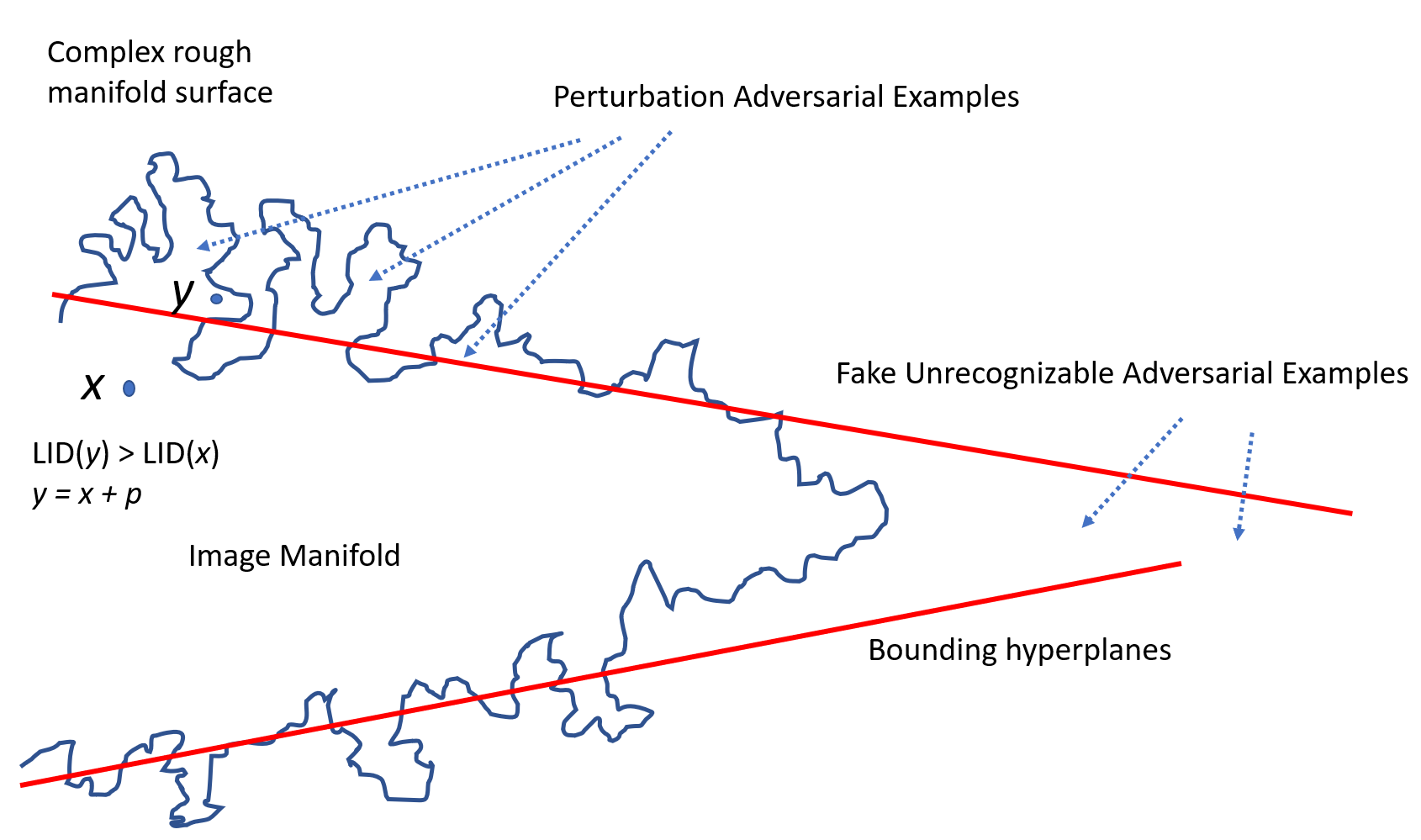}
    \caption{Illustration of how complexity of surfaces of image manifolds can lead to adversarial examples. Space filling ability of the manifold surface as empirically estimated by local intrinsic dimensionality (LID) seems to be very high, see~\cite{lid_paper_deep}. LID is conceptually related to fractal dimension as explained in the text. Theorems~\ref{theorem_adv_negatives} and~\ref{theorem_adv_neg_2} indicate that almost everything is close to surface and empirical findings in~\cite{lid_paper_deep} indicate that surface is topologically quite different from interior. It is likely that LID goes through fractional values from $x$ to $y$. Inability to carve out this surface will lead to adversarial examples.}
    \label{fig_rough_surface}
\end{figure}

{\em Minkowski \mbox{-} Bouligand dimension\/} of a set $X$ is defined as
\[ {\rm dim}_{\rm M}(X) = \mathop{\lim}_{\epsilon \to 0} \frac{\ln{N(X,\epsilon)}}{\ln{1/\epsilon}} \]
by computing number of number of boxes $N(X,\epsilon)$ of side length $\epsilon$ to cover $X$ and can be viewed as a way to compute fractal dimension of $X$.
This can be interpreted as
\[ N(X,\epsilon) \propto \epsilon^{-D} \]
where $D$ is fractal dimension. Consider an $n$-ball $B$ of radius $R$ and let volume of the set $X$ contained inside B be $V$ and let $\epsilon = 1/R$. Then, we can interpret the above as
\[V \propto R^D\]
where $D$ is expansion dimension. Therefore, using two different radii $R_1$ and $R_2$, one can compute $D$ as,
\[ \frac{V_2}{V_1} = \left( \frac{R_2}{R_1} \right)^D \Rightarrow D = \frac{\ln{V_2/V_1}}{\ln{R_2}{R_1}} \]
One can generalize it to probability distributions and make it local at a point $x$, see~\cite{lid_paper_amsaleg, lid_paper_deep}. One considers probability distribution of distances of points from $x$ in local neighborhood of $x$ and takes the probability mass as analogous to volume. Then $D$ can be used as measure of {\em Local Intrinsic Dimensionality} (LID), see~\cite{lid_paper_amsaleg, lid_paper_deep}.

If the surface of the image manifolds is more complex than interior, then we will expect LID to be higher on the surface. It has bee empirically determined in~\cite{lid_paper_deep} that LID of points near surface where adversarial examples exist is significantly higher than for points inside the manifold
\[ {\rm LID} (y) > {\rm LID}(x) \]
where $y = x + p$, for some perturbation $p$. This empirically shows that statistics of natural images leads to complexity of surfaces of image manifolds, see Figure~\ref{fig_rough_surface}. If the dimension also happens to be fractional, then it will indicate roughness in fractal sense. Significant increase in dimension means geometrically some fundamental changes are occurring on the surface. As we move from $x$ to $y$, the manifold starts topologically filling up the space. It is very likely that LID goes through fractional values in this transition which will make one generalize manifolds to arbitrary fractal sets. Even if LID was always integral, surfaces are geometrically complex. Exact characterization of surfaces of image manifolds remains open problem.

Theorems~\ref{theorem_adv_negatives} and~\ref{theorem_adv_neg_2} indicate that we have to worry about surfaces of manifolds as almost everything is close to surface and therefore the root cause of adversarial examples seems to be the complexity of surfaces. Theorems~\ref{theorem_adv_negatives} and~\ref{theorem_adv_neg_2} in this paper along with empirical results in~\cite{lid_paper_deep} explain why adversarial examples exist if the complex surfaces can not be carved out by deep neural networks. We state our proposition now for future work in deep learning, which is self-evident.
\begin{prop}
A machine learning system which can carve out complex non-differentiable manifolds in high-dimensional spaces approximating very closely the ground truth manifolds will rarely suffer from adversarial examples. If trained manifold is identical to ground truth manifold, then there will be no adversarial examples.
\end{prop}
An adversarial example indicates failure of network to generalize and therefore reducing generalization error to 0 eliminates them by definition. Why are we not able to do accurately approximate surfaces with deep networks despite their great success? Deep networks do perform piecewise linear approximation of functions much better than shallow networks due to their depth, see~\cite{bengio_linear}, but this can be improved further. Theoretically, one can not exactly carve out nowhere differentiable manifolds by finitely many neurons. Even if were to consider very close approximation, the surfaces of the manifolds are still too complex and piecewise linear approximation by neural networks with ReLU activation of sizes which are practically feasible at present is not sufficient as evidenced by adversarial examples. Choice of ReLU as activation function is not important in this observation and the problem will persist irrespective of the choice.

In deep networks, we have following difficulties if we want to overcome the problem of adversarial examples.
\begin{enumerate}
\item
First problem we face is that of explosion of size of training data in order to include all possible poses and variations of objects. Very large training datasets will be needed by the present day deep learning.
\item
Even if we can overcome the practical difficulty of getting training data covering all possible poses and variations, to carve out an accurate manifold, we need very large deep networks as the convolutional and max pooling layers provide only limited translation invariance. This is an inherent inefficiency in deep networks in approximating ground truth manifolds~\cite{capsule_networks}.
\item
We employ discriminative approaches rather than generative approaches for classification. Therefore, the goal is to separate out classes based on the training data rather than understanding their poses. There is no concept of poses and other generative parameters.
\item
Though deep learning makes use of hierarchical nature of natural images and learns features from low level to high level, there is no explicit way of implementing parts-whole hierarchy.
\end{enumerate}

It makes us suspect that adversarial examples may be result of above shortcomings of present day deep networks.

\subsection{Parts-Whole Manifold Learning Hypothesis}
How can we potentially do better? In Section~\ref{section_statistics_natural}, we discussed the Manifold Learning Hypothesis. We generalize the hypothesis, inspired by deep learning and by recently proposed Capsule Networks~\cite{capsule_networks}, to include another important feature of natural images, which is their hierarchical nature, in which a scene is made of objects, which are made of parts, and which are made of sub-parts, and so forth. Furthermore, different types of objects can share similar parts. This hierarchical nature of visual scenes manifests itself in the manifold structure. We need a better way to implement pose invariance and for that we need to be able to learn geometric constraints between sub-parts, parts and objects in a data-driven manner. Not only the image manifolds are embedded in low-dimensional subspaces, there is an inherent structure in them which must be algorithmically utilized to make neurons more powerful than neurons in conventional deep networks. This enables us to deal with complex manifolds. Image manifolds are complex but this complexity can be managed effectively through the use of hierarchical inter-relationships between manifolds. We call this {\em Parts-Whole Manifold Learning Hypothesis}.

With capsule networks and an iterative voting algorithm, we can achieve greater pose invariance by exploiting parts-whole relationships in a more robust manner, see~\cite{capsule_networks}, which is a promising step in improving deep learning further. The discussion in this section is inspired by this very recent work by Sabour et al.

Approach like in capsule networks should potentially allow us to work in less number of dimensions. Small parts can be detected using deep networks of low dimensionality. Besides detection, these neurons are trained to perform prediction of pose parameters. These parts can be put together in a whole object more efficiently which incorporates voting based geometric verification of pose of the whole object by its parts. This allows us to work in dimensions which is closer to the theoretical subspace dimensionality as per Manifold Learning Hypothesis. See Figure~\ref{fig_dim_reduction} to get intuition behind how this approach allows us to carve out complex manifolds.

\begin{figure}[t]
    \centering
		\includegraphics[width=\textwidth]{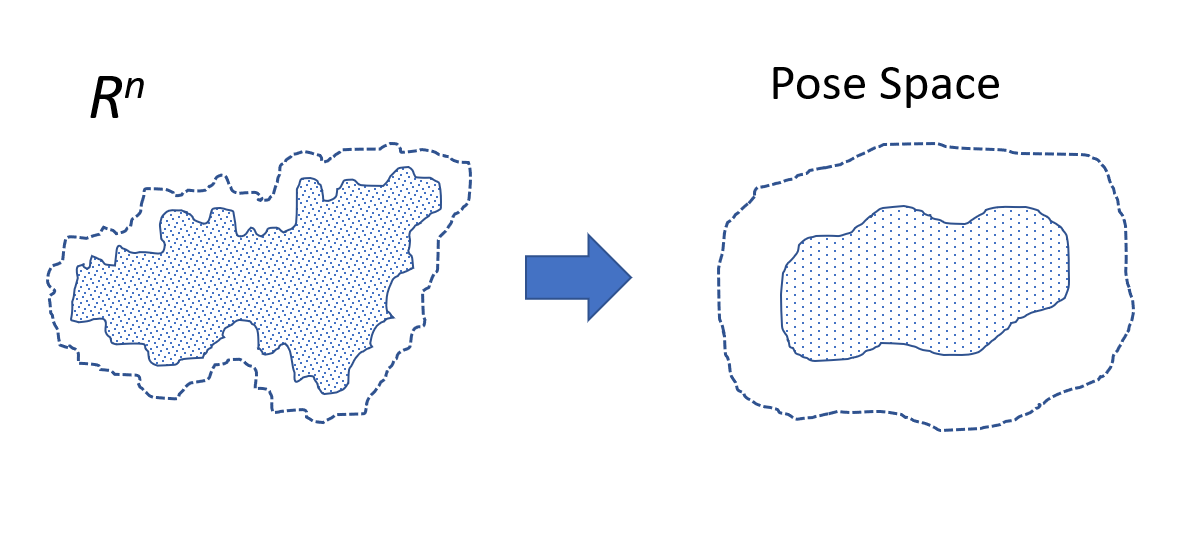}
    \caption{Carving out a complex manifold by keeping dimensionality in check. Manifold contains all images for which ground truth probability is 1. A halo is also shown for borderline images for which ground truth probability is subjectively less than 1 and above a threshold. Neural network learns how to map complex image manifolds into pose space in lower dimensions and will not suffer from adversarial examples.}
    \label{fig_dim_reduction}
\end{figure}

Consider the example of an image of a cat which has a background of trees, sky, clouds, a house and other objects. Where does it exist as a point in cat image manifold and how can we find that point? Consider a generative model as in computer graphics which maps model parameters of each object to rendering of the object in the scene. As these parameters change continuously, you traverse on image manifold of that object, see Figure~\ref{fig_pose_space}. 

For a cat, consider its parts such as eyes, ears, and mouth. Each part is an image manifold of small dimensionality. Detection of each part and its pose parameters identifies a particular point in its manifold. Next higher layer of neural network is trained to map this point to a particular pose of the cat, which is a point on the cat manifold. When points on parts manifolds vote for the same neighborhood in the cat manifold, we can fuse all those evidences to have a combined evidence on the presence and pose of a cat. It can be implemented as an iterative algorithm in which the probability of detecting a cat is refined as one retains only those parts which vote consistently towards one pose and prunes away the spurious and inconsistent ones.

\begin{figure}[b]
    \centering
		\includegraphics[width=\textwidth]{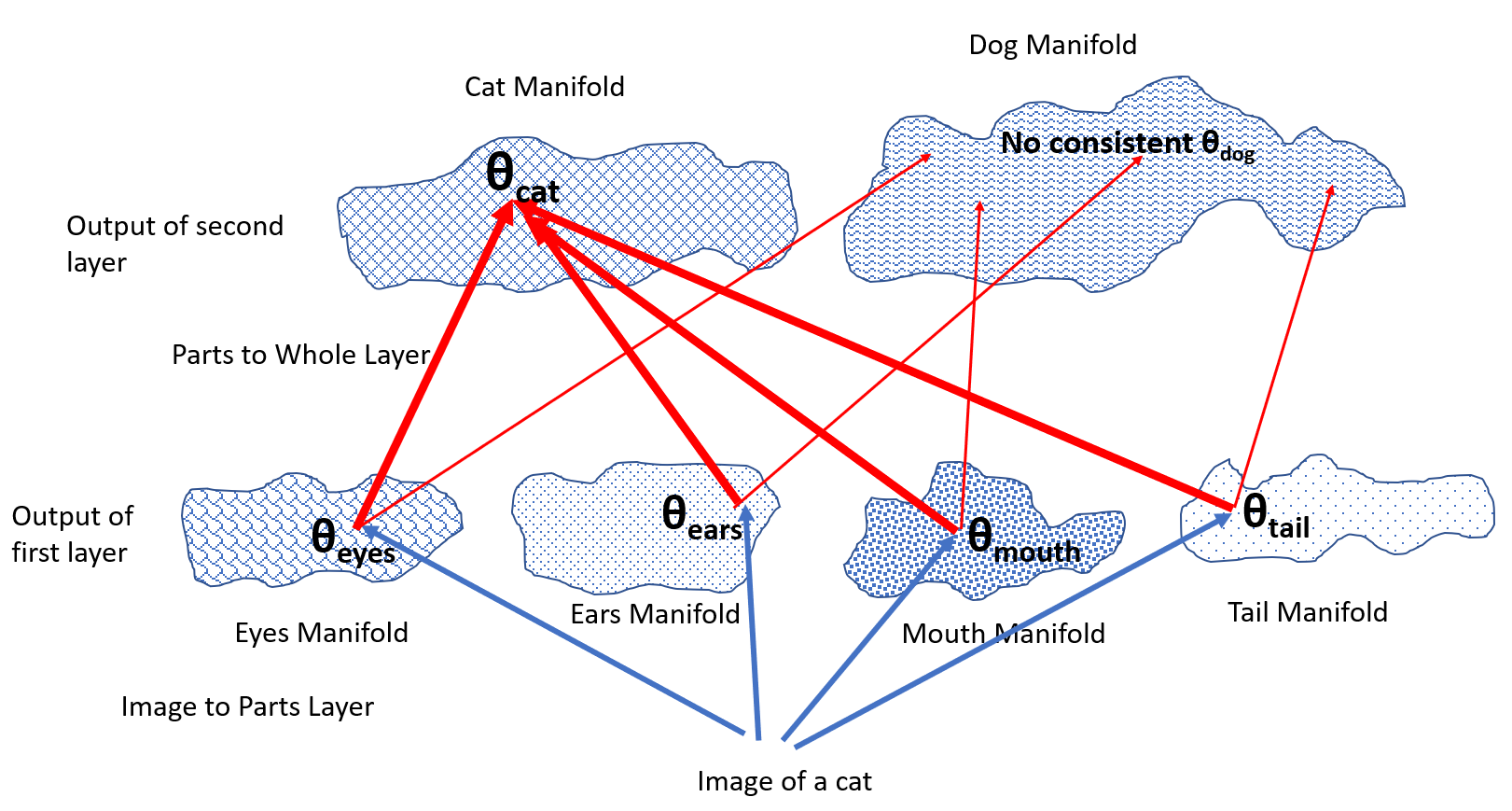}
    \caption{Managing the complexity of image manifolds by making use of Parts-Whole Manifold Learning Hypothesis.}
    \label{fig_future}
\end{figure}

See Figure~\ref{fig_future} for illustration of this approach which can be considered as a conceptual manifold formalization of ideas behind capsule networks~\cite{capsule_networks}. The parameters for parts are mapped to the parameters of the cat by the neural network,
\[\theta_{\rm cat} = f(\theta_{\rm eyes}, \theta_{\rm ears}, \theta_{\rm mouth}, \theta_{\rm tail},\ldots)\]
All the parts vote for a consistent pose of the cat and therefore there is redundancy between them. If there is significant redundancy and mutual information in terms of entropy, we will have
\[|\theta_{\rm cat}| \ll |\theta_{\rm eyes}| + |\theta_{\rm ears}| + |\theta_{\rm mouth}| + |\theta_{\rm tail}| + \ldots\]
This keeps the dimensionality of the whole objects in check though it will increase gradually as their complexity increases. This efficient approach in keeping dimensionality low in every layer of the neural network can be considered as very effective compression of the scene and as combining generative approach with discriminative approach. The generative model ensures that pose parameters are semantically meaningful. Low dimensionality assures us that we are getting close to dimensionality of underlying image manifolds.

Even separate objects which have consistent poses statistically can all work together in agreeing on correct scene interpretation. For example, in a traffic scene for self-driving cars, different objects can consistently vote for the pose of the road.

Objects can be totally unrelated with no mutual information and in that case dimensions will just add. To have both house and cat, the embedded subspace dimensionality will be
\[ |\theta_{\rm scene}| = |\theta_{\rm cat}| + |\theta_{\rm house}| \]
which renders the joint scene of cat with background of house. Note that parameters will also include extra scene parameters, such as viewpoint and lighting, to generatively create a 2-D image from the 3-D world, which may be same for cat and house. So even for seemingly unrelated objects there may be some redundancy as they are part of the same scene.

For future improvements, we should be able to infer these generative parameters and mappings of points from one manifold to another manifold for complex scenes. This can be done using both explicit and implicit approaches.
\begin{enumerate}
\item
We train neurons for hierarchy of parts in which the training data has explicit ground truth for poses of parts and objects. Neurons have regression loss function for these pose parameters, along with discriminative loss and generative loss.
\item
Ground truth is simpler without detailed annotations of parts and relationships. We estimate potential number of parts and pose parameters and design the architecture of the network accordingly. We let the network learn these implicitly through a combination of discriminative loss and generative loss as is done in capsule networks, see~\cite{capsule_networks}.
\end{enumerate}

Such networks will have following strengths.
\begin{enumerate}
\item
The need for getting training data for all poses and variations gets restricted to smallest parts which have low dimensionality and therefore it is practically feasible.
\item
Initial neuron layer needs to predict pose parameters of only fundamental parts from regions of raw images and therefore the number of parameters remain in check once again due to low dimensionality. For higher layers, neurons have to map points from one manifold to another and with ingenious work in future hopefully it will be practically feasible as we keep dimensionality in check at each layer.
\item
The neural network has generative component in form of pose parameters defined either explicitly or implicitly.
\item
The neural network is better implementation of parts-whole manifold learning hypothesis. Therefore, it learns semantically meaningful high-level abstractions rather than superficial and coarse geometries of manifolds.
\end{enumerate}

Successful outcome of the above work should eliminate adversarial examples. If human vision does not suffer from adversarial examples, computer vision should not either.

\section{Conclusion}
In AI community, there has been debate about deep learning. Need for greater rigor and understanding of deep learning has been emphasized. In this paper, we have presented results in the direction of building this rigor and understanding. We have shown how nature of high dimensional spaces explains the working of deep neural networks. We have pointed out fallacy in an argument in a paper published in prior literature which explained adversarial examples. We rigorously explain adversarial examples using properties of image manifolds in high dimensional spaces. We have presented several novel mathematical results explaining adversarial examples, local minima, optimization landscape, image manifolds and properties of natural images. Our mathematical results show that we have to worry about the surfaces of image manifolds as almost everything is close to surface in high dimensions and the root solution of adversarial examples lies in handling the complexity of these surfaces. Exact characterization of surfaces of image manifolds is a topic of further research, and it seems that local dimension goes through a continuum of values including fractional values indicating roughness and space filling properties of surfaces. We also discussed how deep learning can make progress in future. Though high dimensions pose a challenge, we can solve these challenges using novel ways of exploiting characteristics of natural images which will eliminate adversarial examples and overcome the current shortcomings of deep neural networks.

\end{document}